\documentclass[twoside]{article}

\usepackage[accepted]{aistats2021}
\usepackage[round]{natbib}

\usepackage{tabularx}
\usepackage{dsfont}
\usepackage{threeparttable}
\usepackage{amssymb}
\usepackage{mathtools}
\usepackage{caption}
\usepackage{amsthm}
\usepackage{float}
\usepackage{upgreek}
\allowdisplaybreaks
\usepackage{bm}
\usepackage{textcomp}
\usepackage{enumitem}
\usepackage{multirow}
\usepackage{colortbl}
\usepackage[table]{xcolor}
\usepackage{makecell}
\usepackage{microtype}
\usepackage{graphicx}
\usepackage{subfigure}
\usepackage{booktabs}
 \usepackage[linesnumbered,ruled,vlined]{algorithm2e}
\newcommand{\q}{q}

\newcommand{\p}{p}

\newcommand{\e}{E}

\newcommand{\ind}{\perp \!\!\! \perp }

\newcommand{\doo}{\textnormal{do}}
\allowdisplaybreaks

\newtheorem{definition}{Definition}

\newtheorem{lemma}{Lemma}

\usepackage{eucal}

\usepackage{pdfpages}

\begin{document}

\twocolumn[

\aistatstitle{Causal Modeling with Stochastic Confounders}

\aistatsauthor{Thanh Vinh Vo\textsuperscript{1} \And Pengfei Wei\textsuperscript{1} \And  Wicher Bergsma\textsuperscript{2} \And Tze-Yun Leong\textsuperscript{1}}

\runningauthor{Thanh Vinh Vo, Pengfei Wei, Wicher Bergsma, Tze-Yun Leong}

\aistatsaddress{\textsuperscript{1}School of Computing, National University of Singapore\\ \textsuperscript{2}Department of Statistics, London School of Economics and Political Science} ]

\begin{abstract}

This work extends causal inference with stochastic confounders.  We propose a new approach to variational estimation for causal inference based on a representer theorem with a random input space. We estimate causal effects  involving latent confounders that may be interdependent and time-varying from sequential, repeated measurements in an observational study. Our approach extends current work that assumes independent, non-temporal latent confounders, with potentially biased estimators. We introduce a  simple yet elegant algorithm without parametric specification on model components. Our method avoids the need for expensive and careful parameterization in deploying complex models, such as deep neural networks, for causal inference in existing approaches. We demonstrate the effectiveness of our approach on various benchmark temporal datasets.

\end{abstract} %

\section{INTRODUCTION}

The study of causal effects of an intervention or treatment on a specific outcome based on observational data is a fundamental problem in many applications.  Examples include understanding the effects of pollution on health outcomes, teaching methods on a student’s employability, or disease outbreaks on the global stock market. 
A critical problem of causal inference from observational data is confounding. A variable that affects both the treatment and the outcome is known as a confounder of the treatment effects on the outcome.  Standard ways to deal with observable confounders include propensity score-based methods and their variants \citep{rubin2005causal}. However, if a confounder is hidden, the treatment effects on the outcome cannot be directly estimated without further assumptions \citep{pearl2009causal,louizos2017causal}. 
For example, household income, which cannot be easily measured, can affect both the therapy options available to a patient and the health condition after therapy of that patient. %

Recent studies in causal inference \citep{shalit2017estimating,louizos2017causal,madras2019fairness} mainly focus on static data, that is, the observational data has independent and identically distributed (\emph{iid}) noise, and time-independent. In many real-world applications, however, events change over time, e.g., each participant may receive an intervention multiple times and the timing of these interventions may differ across participants. In this case, the time-independent assumption does not hold, and causal inference in the models would degenerate as they fail to capture the nature of time-dependent data. In practice, temporal confounders such as seasonality and long-term trends
can partially contribute to confounding bias. 
For example, soil fertility can be considered as a confounder in crop planting and it may change over time due to different reasons such as annual rainfall or soil erosion. Whenever soil fertility declines (or raises), it would possibly affect the \textit{future} level of fertility. This motivates us to propose stochastic confounders to capture this variable over time. Moreover, if soil fertility is not recorded, it can be considered as a latent confounder. In real-life, there may be many more latent stochastic confounders and may in fact not be interpretable. Thus, these confounders should be taken into account to reduce bias in the estimates of causal effects.

In this work, we introduce a framework to characterize the latent confounding effects over time in causal inference based on the structural causal model (SCM) \citep{pearl2000causality}. Inspired by recent work \citep[e.g.,][]{riegg2008causal,louizos2017causal,madras2019fairness} that handle static, \emph{independent} confounder, we relax this assumption by modeling the confounder as a stochastic process. This approach generalizes the \emph{independent} setting to model confounders that have intricate patterns of interdependencies over time.

Many existing causal inference methods \citep[e.g.,][]{louizos2017causal,shalit2017estimating,madras2019fairness} exploit recent developments of deep neural networks. While effective, the performance of a neural network depends on many factors such as its structure (e.g., the number of layers, the number of nodes, or the activation function) or the optimization algorithm. Tuning a neural network is challenging; different conclusions may be drawn from different network settings. To overcome these challenges, we propose a nonparametric variational method to estimate the causal effects of interest using a kernel as a prior over the reproducing kernel Hilbert space (RKHS).
Our main contributions are summarized as follows:
\begin{itemize}[noitemsep,topsep=0pt,leftmargin=*,wide]
\item We introduce a temporal causal framework with a confounder process that captures the interdependencies of unobserved confounders. This relaxes the independence assumption in recent work  \citep[e.g.,][]{louizos2017causal,madras2019fairness}.
Under this setting, we introduce the concepts of causal path effects and intervention effects, and derive approximation measures of these quantities.
\item Our framework is robust and simple for accurately learning the relevant causal effects: given a time series, learning causal effects quantifies how an outcome is expected to change if we vary the treatment level or intensity. Our algorithm requires no information about how these variables are parametrically related, in contrast to the need of paramterizing a neural network. 
\item We develop a nonparametric variational estimator by exploiting the kernel trick in our temporal causal framework. This estimator has a major advantage: \textit{Complex non-linear functions can be used to modulate the SCM with estimated parameters that turn out to have analytical solutions}. We empirically demonstrate the effectiveness of the proposed estimators.

\end{itemize}
\section{BACKGROUND AND RELATED WORK}
\label{sec:background}

The structural causal model (SCM) of \citet{pearl1995causal,pearl2000causality} is a generic model developed on top of seminal works including structural equation models \citep{goldberger1972structural,goldberger1973structural,duncan1975,duncan2014introduction}, potential outcomes framework of \citet{neyman1923application} and \citet{rubin1974estimating}, and graphical models \citep{pearl23probabilistic}.
The SCM consists of a triplet: a set of exogenous variables whose values are determined by factors outside the framework, a set of endogenous variables whose values are determined by factors within the framework, and a set of structural equations that express the value of each endogenous variable as a function of the values of the other (endogenous and exogenous) variables.
Figure~\ref{fig:latent-confounder}~(a) shows a causal graph with endogenous variables $Y,\, Z,\, W.$ Here $Y$ is the outcome variable, $W$ is the intervention variable, and $Z$ is the confounder variable.
Exogenous variables are variables that are not affected by any other variables in the model, which are not explicitly in the graph.
Causal inference evaluates the effects of an intervention on the outcome, i.e., $\p(Y\,|\,\doo(W=w))$, the distribution of the outcome $Y$ induced by setting $W$ to a specific value $w$. 
Our work focuses on the problem of estimating causal effects, which is different from the works of identifying causal structure \citep{spirtes2000causation}, where  several approaches have been proposed, e.g.,  %
\citet{tian2001causal,peters2013causal,peters2014causal,jabbari2017discovery,huang2019causal}.

\textbf{Estimators with unobserved confounders.} The following efforts take into account the unobserved confounders in causal inference: \citet{montgomery2000measuring,riegg2008causal,de2017proxy,kuroki2014measurement,louizos2017causal,madras2019fairness}.
Specifically, some proxy variables are introduced to replace or infer the latent confounders.
For example, the household income of students is a confounder that affects the ability to afford private tuition and hence the academic performance; it may be difficult to obtain income information directly,
and proxy variables such as zip code, or education level are used instead.
Figures~\ref{fig:latent-confounder}~(b)~and~(c) present two causal graphs used by the recent causal inference algorithms in \citet{louizos2017causal,madras2019fairness}. 
The graphs contain latent confounder $Z$, proxy variable $X$, intervention $W$, outcome $Y$, and observed confounder $S$.

\textbf{Estimators with observed confounders.} \citet{Hill:2011,shalit2017estimating,Alaa:2017,Yao:2018,yoon:2018ganite,kunzel2019metalearners} follow the formalism of potential outcomes and these works do not take into account latent confounders but satisfy the strong ignorability assumption of \cite{rosenbaum1983central}. Of interest is the inference mechanism by \cite{Alaa:2017} where the authors model the counterfactuals as functions living in the reproducing kernel Hilbert space. In contrast, we work on a time series setting within a structural causal framework whose inference scheme directly benefits from the generalization of the empirical risk minimization framework. %

\textbf{Estimators for temporal data.} Little attention has been paid to learning causal effects on non-\emph{iid} setting \citep{guo2018survey,bica2020time}. \citet{lu2018deconfounding} formalized the Markov Decision Process under a causal perspective and the emphasis is mainly on sequential optimization.  While there are mathematical connections to causality, they do not discuss the estimation of average treatment effects using observational data. \citet{li2018estimating} modelled potential outcomes for temporal data with observed confounders using state-space models. The models proposed are linear and quadratic regression. \citet{ning2019bayesian} proposed causal estimates for temporal data with observed confounders using linear models and developed Bayesian flavors inference methods. %
\citet{bojinov2019time} generalized the potential outcomes framework to temporal data and proposed an inference method with Horvitz--Thompson estimator \citep{horvitz1952generalization}. This method, however, does not consider the existence of unobserved confounders. \citet{bica2020Estimating,bica2020time} formalized potential outcomes with observed and unobserved confounders to estimate counterfactual outcomes for treatment plans on each individual where the outcomes are modelled with recurrent neural networks. Several efforts  
\citep[e.g.,][]{kaminski2001evaluating,eichler2005graphical,eichler2007granger,eichler2009causal,eichler2010granger,bahadori2012causality} analysed causation for temporal data based on the notion of Granger causality \citep{granger1980testing}.

\begin{figure}
    \centering
    \includegraphics[width=0.49\textwidth]{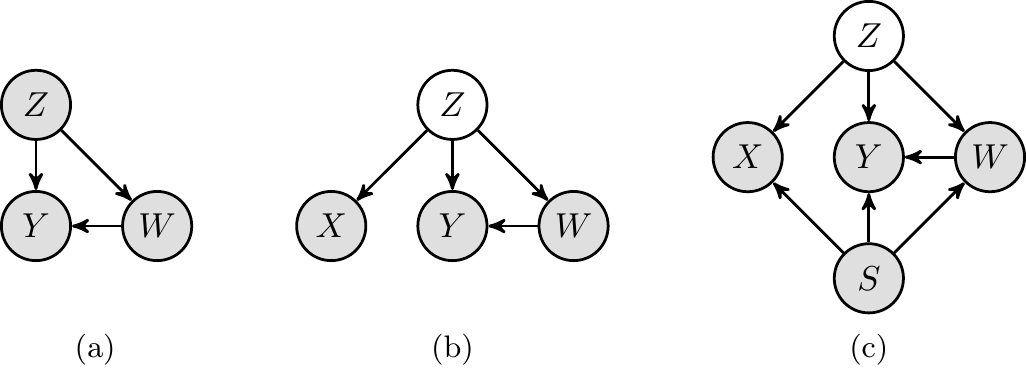}
    \vskip -8pt
    \caption{The SCM framework: approaches of modeling causality. 
    (a) all variables are observed; (b) the confounder $Z$ is latent and being inferred using proxy variable $X$ \citep{louizos2017causal}; (c) there is an additional observed confounder $S$ \citep{madras2019fairness}.}\label{fig:latent-confounder}
    \vskip -6pt
\end{figure}

\begin{figure}
    \centering
\includegraphics[width=0.49\textwidth]{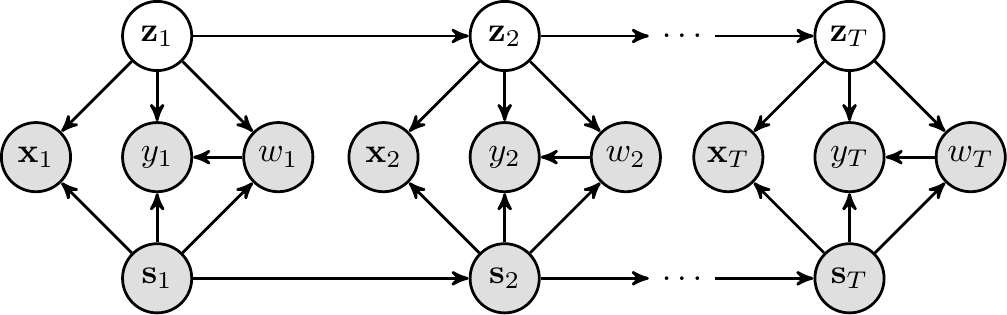}
\vskip -6pt
\caption{A causal model with stochastic confounders for temporal data (Our approach).}\label{fig:model}
\vskip -16pt
\end{figure}

\section{OUR APPROACH}

\label{sec:our-model}
We introduce a temporal causal framework based on SCM, as illustrated in Figure~\ref{fig:model}. Following \citet{Frees:2004longitudinal}, we evaluate the causal effects within a \emph{time interval} denoted by $t$, in concert with the panel data setting. We assume that the interval is large enough to cover the effects of the treatment on the outcome. This assumption is practical as  many interval-censored datasets are recorded monthly or yearly.

\textbf{The latent confounder $\mathbf{z}$.} In real world applications, capturing all the potential confounders is not feasible as some important confounders might not be observed.  When unobserved or latent confounders exist, causal inference from observational data is even more challenging and, as discussed earlier, can result in a biased estimation. The increasing availability of large and rich datasets, however, enables proxy variables for unobserved confounders to be inferred from other observed and correlated variables. In practice, the exact nature and structure of the hidden or latent confounder $\mathbf{z}$ are unknown. %
We assume the structural equation of the latent confounder at time interval $t$ as follows
\begin{align}
\mathbf{z}_t = f_z(\mathbf{z}_{t-1}) + \bm{e}_t,\label{eq:latent-confounders}
\end{align}
where %
the exogenous variable $\bm{e}_t \sim \mathsf{N}(\mathbf{0}, \sigma_z^2\mathbf{I}_{d_z})$,
the function $f_z\colon \mathcal{Z} \mapsto \mathcal{F}_z$ with $\mathcal{Z}$ is the set containing $\mathbf{z}_t$ and $\mathcal{F}_z$ is a Hilbert space, $\sigma_z$ is a hyperparameter and $\mathbf{I}_{d_z}$ denotes the identity matrix. The choice of this structural equation is reasonable because each dimension of $\mathbf{z}_t$ maps to a real value which gives a wide range of possible values for $\mathbf{z}_t$. The function $f_z(\cdot)$ would control the effects of confounder at time interval $t-1$ to its future value at time interval $t$.
With reference to an earlier example, the soil fertility in crop planting can be considered as confounder and this quantity changes over time. When the fertility of soil at time interval $t-1$ declines, which is possibly because of soil erosion, it may result in the fertility of soil at time interval $t$ also declines.
Furthermore, the Gaussian assumption of exogenous variable $\bm{e}_t$ makes it computational tractable for subsequent calculations.

\textbf{The observed variables $y$, $w$, $\mathbf{x}$ and $\mathbf{s}$.} The $d_x$--dimensional observed features at time interval $t$, is denoted by $\mathbf{x}_t\in\mathbb{R}^{d_x}$. Similarly, the treatment variable at time interval $t$ is given by $w_t$. $y_t$ and $\mathbf{s}_t$ denote the outcome and the observed confounder at time interval $t$, respectively. 
$\mathbf{z}_t$ denotes an unobserved confounder. Let $\mathcal{Y}$, $\mathcal{W}$, $\mathcal{S}$, $\mathcal{X}$ be sets containing $y_t$, $w_t$, $\mathbf{s}_t$, $\mathbf{x}_t$, respectively, and let $\mathcal{F}_y$, $\mathcal{F}_w$, $\mathcal{F}_s$, $\mathcal{F}_x$ be Hilbert spaces. Let $f_y\colon \mathcal{Z}\times\mathcal{W}\times\mathcal{S} \mapsto \mathcal{F}_y$, $f_w\colon \mathcal{Z}\times\mathcal{S} \mapsto \mathcal{F}_w$, $f_s\colon \mathcal{S} \mapsto \mathcal{F}_s$, $f_x\colon\mathcal{Z}\times\mathcal{S} \mapsto \mathcal{F}_x$.
We postulate the following structural equations:%
\begin{align}
\!\!\!\!\mathbf{s}_t &= f_s(\mathbf{s}_{t-1}) + \bm{o}_t, &y_t &= f_y(w_t, \mathbf{z}_t, \mathbf{s}_t) + v_t,\label{eq:covariate}\\
\!\!\!\!\mathbf{x}_t &= f_x(\mathbf{z}_t, \mathbf{s}_t) + \bm{r}_t,  &w_t &= \mathds{1}(\varphi(f_w(\mathbf{z}_t, \mathbf{s}_t))\ge u_t),\label{eq:treatments}
\end{align}
where the exogenous variables $\bm{o}_t \sim \mathsf{N}(\mathbf{0}, \sigma_s^2\mathbf{I}_{d_s})$, $\bm{r}_t \sim \mathsf{N}(\mathbf{0}, \sigma_x^2\mathbf{I}_{d_x})$, $v_t \sim \mathsf{N}(0, \sigma_y^2)$ and $u_t \sim \mathsf{U}[0, 1]$, $\mathds{1}(\cdot)$ denotes the indicator function
and $\varphi(\cdot)$ is the logistic function. The last structural equation implies that $w_t$ is Bernoulli distributed given $\mathbf{z}_t$ and $\mathbf{s}_t$, and $\p(w_t=1|\mathbf{z}_t,\mathbf{s}_t) = \varphi(f_w(\mathbf{z}_t, \mathbf{s}_t))$. Similar reasoning holds for these assumptions in that they afford computational tractability in terms of time series formulation. Our model assumes that the treatment and outcome at time interval $t-1$ are independent with the treatment and outcome at time interval $t$ given the confounders. These assumptions are important in real-life, for example, in the context of agriculture, the crop yield (outcome) of the current crop season probably does not affect crop yield in the next season, and similarly for the chosen of fertilizer (treatment). However, they may be correlated through the latent confounders, e.g., soil fertility of the land.

\textbf{Learnable functions.}
With Eqs.~(\ref{eq:latent-confounders})-(\ref{eq:treatments}), we need to learn $f_i$, where $i\in \bm{\mathcal{A}} \vcentcolon=\{y,w,z,x,s\}$. Standard methods to model these functions include linear models or multi-layered neural networks. Selecting models for these functions relies on many problem-specific aspects such as types of data (e.g., text, images), dataset size (e.g., hundreds, thousands, or millions of data points), and data dimensionality (high- or low-dimension). 
We propose to model these functions through an \textit{augmented representer theorem}, to be desribed in Section~\ref{sec:estimating}.
\subsection{Causal Quantities of Interest}
\label{sec:rep-causal-effect}
Temporal data capture the evolution of the characteristics over time. Based on earlier work \citep{louizos2017causal,pearl2009causal,madras2019fairness}, we evaluate causal inference from temporal data by assuming confounders satisfy Eq.~(\ref{eq:latent-confounders}). The corresponding causal graph is shown in Figure~\ref{fig:model}. %
 This relaxes the independent assumption in \citet{louizos2017causal} and \citet{madras2019fairness}. 
 In particular, we aim to measure the causal effects of $w_t$ on $y_t$ given the covariate $\mathbf{x}_t$, where $\mathbf{x}_t$ serves as the proxy variable to infer latent confounder $\mathbf{z}_t$. %
 This formulation subsumes earlier approaches \citep{louizos2017causal,madras2019fairness}. Considering multiple time intervals, we further denote the vector notations for $T$ time intervals as follows:
$\mathbf{y} = [y_1,\!..., y_T]^\top, \mathbf{w} = [w_1,\!..., w_T]^\top, \mathbf{s} = [\mathbf{s}_1,\!..., \mathbf{s}_T]^\top, 
\mathbf{x} = [\mathbf{x}_1,\!..., \mathbf{x}_T]^\top, 
\mathbf{z} = [\mathbf{z}_1,\!..., \mathbf{z}_T]^\top, 
\mathbf{z}_0 = \mathbf{0}, \mathbf{s}_0 = \mathbf{0}$. %
We define \textit{fixed-time} causal effects as the causal effects at a time interval $t$ and \textit{range-level} causal effects as the average causal effects in a time range. The time and range-level causal effects can be estimated using Pearl's \textit{do}-calculus.
We model the unobserved confounder processes using the latent variable $\mathbf{z}_t$, inferred for each observation ($\mathbf{x}_t, \mathbf{s}_t, w_t, y_t$) at time interval $t$. The interval is assumed to be large enough to cover the effects of the treatment $w_t$ on the outcome $y_t$. This assumption is practical as  many interval-censored datasets are recorded monthly or yearly.
\begin{definition}\label{def:ep} Let $\mathbf{w}_1$ and $\mathbf{w}_2$ be two treatment paths. The treatment effect path (or effect path) of $t \in [1,2,\!...,T]$ is defined as follows
	\begin{align}
	\!\!\!\textnormal{EP} \vcentcolon= \e[\,\mathbf{y}\,|\,\doo(\mathbf{w} \!=\! \mathbf{w}_1), \mathbf{x}] \!-\! \e[\,\mathbf{y}\,|\,\doo(\mathbf{w} \!=\! \mathbf{w}_2), \mathbf{x}].\label{eq:effect-path}
	\end{align}
\end{definition}
The effect path (EP) is a collection of causal effects at time interval $t \in [1,2,\dots,T]$.
\begin{definition}\label{def:ae} Let $\textnormal{EP}$ be the effect path satisfying Definition~\textnormal{\ref{def:ep}}. Let $\textnormal{EP}_t \in \textnormal{EP}$ be the effect at time interval $t$. The average treatment effect (ATE) in $[T_1, T_2]$ is defined as follows
	\begin{align}
	\textnormal{ATE} \textstyle \vcentcolon= \left(\sum_{t=T_1}^{T_2} \textnormal{EP}_t\right)\big/(T_2-T_1+1).\label{eq:avg-effect}
	\end{align}
\end{definition}
The average treatment effect (ATE) quantifies the effects of a treatment path $\mathbf{w}_1$ over an alternative treatment path $\mathbf{w}_2$. This work focuses on evaluating the causal effects with binary treatment.%
~The key quantity in estimating the effect path and average treatment effects is $\p(\mathbf{y} \,|\, \doo(\mathbf{w}), \mathbf{x})$, which is the distribution of $\mathbf{y}$ given $\mathbf{x}$ after setting variable $\mathbf{w}$ by an intervention. 
Following Pearl's back-door adjustment and invoking the properties of $d$-separation, the causal effects of $\mathbf{w}$ to $\mathbf{y}$ given $\mathbf{x}$ with respect to the causal graph in Figure~\ref{fig:model} is as follows
	\begin{eqnarray}
	\label{eq:tmp}
	\p(\mathbf{y} | \doo(\mathbf{w}), \mathbf{x}) = \int\p(\mathbf{y} | \mathbf{w}, \mathbf{z}, \mathbf{s})\p(\mathbf{z},\mathbf{s}| \mathbf{x})d\mathbf{s}d\mathbf{z}.\label{eq:causal-effect}
	\end{eqnarray}
The expression in Eq.~(\ref{eq:tmp}) typically does not have an analytical solution when the distributions $\p(\mathbf{y} \,|\, \mathbf{w}, \mathbf{z}, \mathbf{s})$, $\p(\mathbf{z},\mathbf{s}\,|\, \mathbf{x})$ are parameterized by more involved impulse functions, e.g., a nonlinear function such as multi-layered neural network. Thus, the empirical expectation of $\mathbf{y}$ is used as an approximation. 
To do so, we first draw samples of $\mathbf{z}$ and $\mathbf{s}$ from $\p(\mathbf{z},\mathbf{s}| \mathbf{x})$, and then substitute these samples to $\p(\mathbf{y}| \mathbf{w}, \mathbf{z}, \mathbf{s})$ to draw samples of $\mathbf{y}$. The whole procedure is carried out using forward sampling techniques under specific forms of $\p(\mathbf{y}| \mathbf{w}, \mathbf{z}, \mathbf{s})$, $\p(\mathbf{z}|\mathbf{x},\mathbf{s},\mathbf{w},\mathbf{y})$, $\p(\mathbf{y}|\mathbf{s},\mathbf{x},\mathbf{w})$, $\p(\mathbf{w}|\mathbf{s},\mathbf{x})$ and $\p(\mathbf{s}|\mathbf{x})$. In the next section, we present approximations of these probability distributions.

\section{ESTIMATING CAUSAL EFFECTS}
\label{sec:estimating}
Estimating causal effects requires systematic sampling from the following distributions: $\p(\mathbf{s}|\mathbf{x})$, $\p(\mathbf{w}|\mathbf{s},\mathbf{x})$, $\p(\mathbf{y}|\mathbf{s},\mathbf{x},\mathbf{w})$, $\p(\mathbf{z}|\mathbf{x},\mathbf{s},\mathbf{w},\mathbf{y})$ and  $\p(\mathbf{y}|\mathbf{w},\mathbf{z},\mathbf{s})$. This section presents approximations to these distributions.

\subsection{The Posterior of Latent Confounders} Exact inference of $\mathbf{z}$ is intractable for many models, such as multi-layered neural networks. Hence, we infer $\mathbf{z}$ using variational inference, which approximates the true posterior $\p(\mathbf{z}|\mathbf{x},\mathbf{s},\mathbf{w},\mathbf{y})$ by a parametric variational posterior $\q(\mathbf{z}|\mathbf{x},\mathbf{s}, \mathbf{w},\mathbf{y})$. This approximation is obtained by minimizing the Kullback-Leibler divergence ($D_{\textrm{KL}}$): $D_{\textrm{KL}}[\q(\mathbf{z}|\mathbf{x},\mathbf{s},\mathbf{w},\mathbf{y})\|\p(\mathbf{z}|\mathbf{x},\mathbf{s}, \mathbf{w},\mathbf{y})]$, which is equivalent to maximizing the evidence lower bound (ELBO) of the marginal likelihood:
\begin{align}
\!\!\!\!\!\!\mathcal{L} &\!=\!E_{\mathbf{z}}\Big[\log\p(\mathbf{y}, \mathbf{w}, \mathbf{s}, \mathbf{x} | \mathbf{z})\Big]  \label{eq:elbo}\\
&\!-\! \sum_{t=1}^TE_{\mathbf{z}_{t-1}}\Big[D_{\text{KL}}\big(\q(\mathbf{z}_t|y_t,\mathbf{x}_t,w_t,\mathbf{s}_t)\|\p(\mathbf{z}_t|\mathbf{z}_{t-1})\big)\Big].\nonumber
\end{align}
The expectations are taken with respect to variational posterior $\q(\mathbf{z}|\mathbf{x},\mathbf{s}, \mathbf{w},\mathbf{y})$, and each term in the ELBO depends on the assumption of their distribution family presented in Section~\ref{sec:our-model}. %
We further assume that the variational posterior distribution takes the form $
\q(\mathbf{z}_t|\cdot) = \mathsf{N}(\mathbf{z}_t|f_q(y_t, w_t, \mathbf{s}_t, \mathbf{x}_t), \sigma_q^2\mathbf{I}_{d_z})$,
where $f_q\colon \mathcal{Y}\times\mathcal{W}\times\mathcal{S} \times \mathcal{X} \rightarrow \mathcal{F}_z$ is a function parameterizing the designated distribution, $\sigma_q$ is a hyperparameter, and $\mathbf{I}_{d_x}$ denotes the identity matrix. Our aim is to learn $f_i$ where $i\in \bm{\mathcal{A}}\leftarrow\bm{\mathcal{A}} \cup \{q\}$.

\subsubsection{Inference of The Posterior}
\label{sec:elbo-rt}

To formulate a regularized empirical risk, we draw $L$ samples of $\mathbf{z}$ from the variational posterior  using reparameterization trick: $\mathbf{z}^l = [\mathbf{z}_1^l,\!...,\mathbf{z}_T^l]$ with $\mathbf{z}_t^l = f_q(y_t, w_t, \mathbf{s}_t, \mathbf{x}_t) + \sigma_q\bm{\varepsilon}_t^l$ and $\bm{\varepsilon}_t^l \sim \mathsf{N}(0,\mathbf{I}_{d_x})$. By drawing $L$ noise samples $\bm{\varepsilon}_t^1,\!...,\bm{\varepsilon}_t^L$ at each time interval $t$ in advance, we obtain a \textit{complete} dataset
\begin{align*}
\mathcal{D} = \bigcup_{l=1}^L	\bigcup_{t=1}^T\Big\{\left(y_t, w_t, \mathbf{x}_t, \mathbf{s}_t, \mathbf{z}_t^l\right)\Big\}.
\end{align*}
At each time interval $t$, the dataset gives a tuple of the observed values $y_t, w_t, \mathbf{x}_t, \mathbf{s}_t$, and an expression of $\mathbf{z}_t^l = f_q(y_t, w_t, \mathbf{s}_t, \mathbf{x}_t) + \sigma_q\bm{\varepsilon}_t^l$. We state the following:

\begin{lemma} \label{theo:ker-min-form}
	Let $\kappa_{i}$ be kernels and $\mathcal{H}_i$  their associated reproducing kernel Hilbert space (RKHS), where $i\in\bm{\mathcal{A}}$. Let the empirical risk obtained from the negative ELBO be $\widehat{\mathcal{L}}$. Consider minimizing the following objective function
\begin{align}
	J &=\widehat{\mathcal{L}}\left({\bigcup}_{i\in \bm{\mathcal{A}}} f_i\right)  + \sum_{i\in\bm{\mathcal{A}}}\lambda_i\|f_i\|_{\mathcal{H}_i}^2\label{eq:loss}
\end{align}
	with respect to functions $f_i$ ($i\in\bm{\mathcal{A}}$), where $\lambda_i \in \mathbb{R}^+$. Then, the minimizer of \textnormal{(\ref{eq:loss})} has the following form $f_i = \sum_{l=1}^{T\times L}\,\kappa_i(\,\cdot\,,\bm{\nu}^i_l)\bm{\beta}_{l}^i$ ($\forall i\in\bm{\mathcal{A}}$), where $\bm{\nu}^i_l$ is the $l^{th}$ input to function $f_i$, i.e., it is a subset of the $l^{th}$ tuple of $\mathcal{D}$, and  the coefficients $\bm{\beta}_{l}^i$ are vectors in the Hilbert space $\mathcal{F}_i$. This minimizer further emits the following solution: $\bm{\beta}^i = [\bm{\beta}_l^i,\!..., \bm{\beta}_{TL}^i]^\top  =\big(\sum_{l=1}^L{\mathbf{K}_i^l}^\top\mathbf{K}_i^l + \lambda_iL\mathbf{K}_i\big)^{-1}\sum_{l=1}^L{\mathbf{K}_i^l}^\top\bm{\psi}^i$ for $i\in \bm{\mathcal{A}}\setminus\{w,q\}$ and $\bm{\psi}^y = \mathbf{y}$, $\bm{\psi}^x = \mathbf{x}$, $\bm{\psi}^s = \mathbf{s}$, $\bm{\psi}^z = \mathbf{K}_q\bm{\beta}^q$.%
\end{lemma}
As mentioned in Section~\ref{sec:our-model}, we propose to model $f_i$ using kernel methods. A natural way is to develop a slight modification to the classical representer theorem  \citep{kimeldorf1970correspondence,scholkopf2001generalized} so that it can be applied to the optimization of the ELBO. Eq.~(\ref{eq:loss}) is minimized with respect to the weights $\bm{\beta}^i$ and hyperparameters of the kernels. The proof of Lemma \ref{theo:ker-min-form} is deferred to Appendix.

\begin{lemma}
\label{lem:convex-analysis}
    For any fixed $\bm{\beta}^q$, the objective function $J$ in Eq.~\emph{(\ref{eq:loss})} is convex with respect to $\bm{\beta}^i$ for all $i\in \bm{\mathcal{A}}\setminus\{q\}$.
\end{lemma}

\begin{proof}
We present the sketch of proof here and details are deferred to Appendix. It can be shown that the objective function is a combination of several components including $(\bm{\beta}^i)^\top \mathbf{C}\bm{\beta}^i$, $ \mathbf{c}^\top\bm{\beta}^i$, $-\mathbf{w}^\top\log\varphi(\mathbf{K}_w^l\bm{\beta}^w)$ and $- (\bm{1}-\mathbf{w})^\top\log\varphi(-\mathbf{K}_w^l\bm{\beta}^w)$, where $i\in\{y,s,x,z\}$, $\mathbf{C}$ is a positive semi-definite matrix and $\mathbf{c}$ is a vector. The first component is a quadratic form. Thus, its second-order derivative with respective to $\bm{\beta}^i$ is a positive semi-definite matrix; hence, it is convex. The second term is a linear function of $\bm{\beta}^i$ thus it is convex. The two last terms come from cross-entropy loss function and the input to these function are linear combination of the kernel function evaluated between each pair of data points. So these two terms are also convex. 
Consequently, $J$ is convex with respective to $\bm{\beta}^i$ ($i\in\bm{\mathcal{A}}\setminus\{q\}$) because it is a linear combination of convex components.
\end{proof}
\vspace{-6pt}

Lemma~\ref{lem:convex-analysis} implies that at an iteration in the optimization that $\bm{\beta}^q$ reach its convex hull, the objective function $J$ will reach its minimal point after a few more iterations. This is because the non-convexity of $J$ is induced only by $\bm{\beta}^q$. This result indicates that we should attempt different random initialization on $\bm{\beta}^q$ instead of the other parameters when optimizing $J$ because $\bm{\beta}^i$ ($i\in\bm{\mathcal{A}}\setminus\{q\}$) always converge to its optimal point (conditioned on $\bm{\beta}^q$).

\subsection{The Auxiliary Distributions} 

The previous steps approximate the posterior $\p(\mathbf{z}|\cdot)$ by variational posterior $\q(\mathbf{z}|\cdot)$ and estimate the density of $\p(\mathbf{y}|\mathbf{w},\mathbf{z},\mathbf{s})$. This section outlines the approximation of $\p(\mathbf{y}|\mathbf{s},\mathbf{x},\mathbf{w})$, $\p(\mathbf{w}|\mathbf{s},\mathbf{x})$ and $\p(\mathbf{s}|\mathbf{x})$. Denote their corresponding approximation as $\widetilde{\p}(\mathbf{y}|\mathbf{s},\mathbf{x},\mathbf{w})$, $\widetilde{\p}(\mathbf{w}|\mathbf{s},\mathbf{x})$, $\widetilde{\p}(\mathbf{s}|\mathbf{x})$, we estimate parameters of those distribution directly using classical representer theorem. We briefly describe how to approximate $\widetilde{\p}(\mathbf{w}|\mathbf{s},\mathbf{x})$. The regularized empirical risk obtained from the negative log-likelihood of $\widetilde{\p}(\mathbf{w}|\mathbf{s},\mathbf{x})$ is $
J_w = \sum_{t=1}^T \ell_{\textnormal{Xent}}\big(w_t, g_w(w_{t-1},\mathbf{s}_t,\mathbf{x}_t)\big) + \delta_w\|g_w\|_{\mathcal{V}_w}^2$,
where $\ell_{\textnormal{Xent}}(\cdot,\cdot)$ is the cross-entropy loss function, $\delta_w \in \mathbb{R}^+$, $g_w\colon\mathcal{W}\times \mathcal{S} \times \mathcal{X}\mapsto \mathcal{F}_w$, and $\mathcal{V}_w$ is the RKHS with kernel function $\tau_w(\cdot,\cdot)$. By classical representer theorem, the minimized form of $g_w$ is given by $g_w = \sum_{j=1}^T \alpha_j^w\tau_w(\,\cdot\,,[w_{j-1}, \mathbf{s}_j, \mathbf{x}_j])$, where $\alpha_j^w \in \mathbb{R}$ is the parameter to be learned. It can be shown that $J_w$ is a convex objective function because the input to the cross-entropy function is linear. Other distributions can be estimated in a similar fashion and the full description is deferred to Appendix.

\section{EXPERIMENTS}
\label{sec:experiment}
In this section, we examine the performance of our framework in estimating causal effects from temporal data on both synthetic and real-world datasets. We compare with the following baselines: \textbf{(1)~} The potential outcomes-based model for time series data by \citet{bojinov2019time}, which uses Horvitz-Thompson estimator to evaluate causal effects. The key factor of this method is the `adapted propensity score' to which we have implemented two versions of this score. The first one uses a fully connected neural network.
Herein, we assume that $
	\p(w_t|\mathbf{w}_{1:t-1}, \mathbf{y}_{1:t-1}) \!=\! \mathsf{Bern}(w_t|f(w_{t-1},y_{t-1}))$
 with $f(w_{t-1},y_{t-1})$  is a neural network taking the observed value of $w_{t-1},y_{t-1}$ as input to predict $w_t$. 
The second one uses Long-Short Term Memory (LSTM) to estimate $\p(w_t|\mathbf{w}_{1:t-1}, \mathbf{y}_{1:t-1})$. 
\textbf{(2)~}The second baseline is {TARNets}, a model for inferring treatment effect by \citet{shalit2017estimating}.
\textbf{(3)~}The third baseline, {CEVAE} \citep{louizos2017causal} is a causal inference model based on variational auto-encoders. We reused the code of CEVAE and TARNets which are available online to train these models. \textbf{(4)~}The last baseline is fairness through causal awareness by \citet{madras2019fairness}. This method is an extension of CEVAE where they consider two types of confounder: observed and latent ones. To evaluate the performance of each methods, we report the absolute error of the ATE: $\epsilon_\textrm{ATE} \vcentcolon= |\,\textrm{ATE} - \widehat{\textrm{ATE}}\,|$ and the precision of estimating heterogeneous effects (PEHE) \citep{Hill:2011}: $\epsilon_\textrm{PEHE} \vcentcolon= \big(\sum_{i=T_1}^{T_2}(\textrm{EP}_i - \widehat{\textrm{EP}}_i)^2\big)/(T_2-T_1+1)$.%

For the neural network setup on each baseline, we closely follow the architecture of \citet{louizos2017causal} and \citet{shalit2017estimating}. Unless otherwise stated, we utilize a fully connected network with \texttt{ELU} as activation function and use the same number of hidden nodes in each hidden layer. We fine-tune the network with $2, 4, 6$ hidden layers and $50, 100, 150, 200, 250$ nodes per layer. We also fine-tune the learning rate in $\{10^{-1}, 10^{-2}, 10^{-3}, 10^{-4}\}$ and use \textit{Adamax} \citep{DBLP:journals/corr/KingmaB14} for optimization.

\subsection{Illustration on Modeling with Stochastic Confounders}
Before carrying out our main experiments, we first illustrate the importance of our proposed stochastic confounders in estimating causal effects for time series data. We consider the ground truth causal model whose structural equations are as follows:
\begin{align*}
    s_t &= a_0 + a_1s_{t-1} + o_t,
    &\,\,\,\,\,\,\,\,\,\,\,y_t = f_y(s_t, w_t) + v_t,\\
    w_t &= \mathds{1}(\varphi(b_0 + b_1z_t) \ge u_t),&
\end{align*}
where $\mathds{1}(\cdot)$ is the indicator function, $s_0 = 0$, $o_t \sim \mathsf{N}(0,0.3^2)$, $u_t \sim \mathsf{U}[0,1]$, and $v_t \sim \mathsf{N}(0,1)$. In this model, $s_t$, $w_t$, $y_t$ are endogenous variables, and  $o_t$, $u_t$, $v_t$ are exogenous variables. The functions $f_i$ ($i \in \bm{\mathcal{A}}\setminus\{y\}$) in this model are linear. We consider three cases for the ground truth of function $f_y$: \textbf{(1)} Linear outcome: $f_y(s_t, w_t) = c_0 + c_1s_t + c_2w_t$, \textbf{(2)} Quadratic outcome: $f_y(s_t, w_t) = (c_0 + c_1s_t + c_2w_t)^2$, and \textbf{(3)} Exponential outcome: $f_y(s_t, w_t) = \exp(c_0 + c_1s_t + c_2w_t)$. For all the cases, we randomly choose the true parameters $(a_0, a_1, b_0, b_1, c_0, c_1, c_2) = (0.7, 0.95, 0.2, -0.1, 0.7, 0.4, 1.7)$ and sample three time series of length $T=1000$ for $z_t, w_t, y_t$ from the above distributions. Herein, we keep $w_t, y_t, s_t$ as the observed data and use the following three inference models to estimate causal effects: Model-1: without confounders, Model-2: with \emph{iid} confounders and Model-3: with stochastic confounders. The error reported in Table~\ref{tab:illustration-stochastic-confounders} shows that the model with stochastic confounders outperforms the others as it fits well to the data. In the following sections, we present our main experiments with more complicated functions $f_i$ ($i \in \bm{\mathcal{A}}$).
\begin{table*}
\caption{The errors of the estimated treatment effects}\label{tab:illustration-stochastic-confounders}
\vskip -9pt
\centering
	\setlength{\tabcolsep}{2.7pt}
\begin{tabular}{l>{\centering\arraybackslash}p{1.6cm}>{\centering\arraybackslash}p{1.6cm}>{\centering\arraybackslash}p{1.6cm}>{\centering\arraybackslash}p{1.6cm}>{\centering\arraybackslash}p{1.9cm}>{\centering\arraybackslash}p{1.9cm}}
\toprule
  & \multicolumn{2}{c}{\begin{tabular}[c]{@{}c@{}}Model-1 \\without confounders\end{tabular}} & \multicolumn{2}{c}{\begin{tabular}[c]{@{}c@{}} Model-2\\with \emph{iid} confounders \end{tabular}}& \multicolumn{2}{c}{\begin{tabular}[c]{@{}c@{}}Model-3\\with stochastic confounders\end{tabular}} \\\cmidrule(lr){2-3}\cmidrule(lr){4-5}\cmidrule(lr){6-7}
              & $\sqrt{\epsilon_\textrm{PEHE}}$            & $\epsilon_\textrm{ATE}$                    & $\sqrt{\epsilon_\textrm{PEHE}}$                 & $\epsilon_\textrm{ATE}$                     & $\sqrt{\epsilon_\textrm{PEHE}}$            & $\epsilon_\textrm{ATE}$             \\\cmidrule(lr){2-3}\cmidrule(lr){4-5}\cmidrule(lr){6-7}
Linear outcome                & 0.06$\pm$.01            & 0.06$\pm$.01        & \textbf{0.05$\pm$.01}        & \textbf{0.05$\pm$.01}     & \textbf{0.05$\pm$.01} &    \textbf{0.05$\pm$.01}\\
Quadratic outcome             & 2.53$\pm$.08            & 2.26$\pm$.06       & 1.78$\pm$.03                 & 0.88$\pm$.03     & \textbf{0.33$\pm$.02}   & \textbf{0.26$\pm$.02}\\
Exponential outcome            & 4.52$\pm$.12            &3.36$\pm$.15        & 4.18$\pm$.20                 & 2.41$\pm$.07     & \textbf{0.91$\pm$.03}    & \textbf{0.78$\pm$.06}\\ \bottomrule

\end{tabular}
\vskip -6pt
\end{table*}

\subsection{Synthetic Experiments}
Since obtaining ground truth for evaluating causal inference methods is challenging, most of the recent methods are evaluated with synthetic or semi-synthetic datasets \citep{louizos2017causal}. This set of experiments is conducted on four synthetic datasets and one benchmark dataset, where three synthetic datasets are time series with latent stochastic confounders and the other two datasets are static data with \textit{iid} confounders.
 
\textbf{Synthetic dataset.} We simulate data for $y_t$, $w_t$, $\mathbf{x}_t$, $\mathbf{s}_t$ and $\mathbf{z}_t$ from their corresponding distributions as in Eqs.~(\ref{eq:latent-confounders})-(\ref{eq:treatments}) each with length $T=200$.
The ground truth nonlinear functions $f_i ( i \in \bm{\mathcal{A}}\setminus\{q\})$ with respect to the distributions of $y_t, w_t, \mathbf{x}_t, \mathbf{s}_t, \mathbf{z}_t$ are fully connected neural networks (refer to Appendix for details of these functions).
Using different numbers of the hidden layers, i.e., 2, 4, and 6, we construct three synthetic datasets, namely TD2L, TD4L, and TD6L.
For these three datasets, we sample the latent confounder variable $\mathbf{z}$ satisfying Eq.~(\ref{eq:latent-confounders}).
We also construct another dataset, TD6L-iid, that uses 6 hidden layers but with the \emph{iid} latent confounder variable  $\mathbf{z}$, i.e., $\mathbf{z}_t \ind \mathbf{z}_s, \forall t,s$. Finally, we only keep $y_t$, $w_t$, $\mathbf{x}_t$, $\mathbf{s}_t$ as observed data for training. Due to limited space, details of the simulation is presented in Appendix.

\textbf{Benchmark dataset.}  
Infant Health and Development Program (IHDP) dataset \citep{Hill:2011} is a study on the impact of specialist visits on the cognitive development of children. This dataset has 747 entries, each with 25 covariates that describe the properties children. The treatment group consists of children who received specialist visits and a control group that includes children who did not. For each child, a treated and a control outcome were simulated from the numerical scheme of the NPCI library \citet{dorie2016npci}. 
\ifdefined\arXiv
\begin{table}[!ht]
	\caption{$\epsilon_\textrm{ATE}$ of each method on different datasets (lower is better).}
	\vskip 3pt
	\label{tab:mae}
	\centering
	\setlength{\tabcolsep}{7.8pt}
	\begin{tabular}{lccccc}
		
		\toprule
		\multicolumn{1}{c}{\multirow{2}{*}{Method}} & \multicolumn{3}{c}{\begin{tabular}[c]{@{}c@{}}Temporal Data (latent stochastic confounders)\end{tabular}} & \multicolumn{2}{c}{\begin{tabular}[c]{@{}c@{}}IID Data  (\emph{iid} confounders)\end{tabular}} \\ \cmidrule(lr){2-4}\cmidrule(lr){5-6}
		\multicolumn{1}{c}{}                                 & TD2L                   & TD4L                  & TD6L                  & TD6L-iid& IHDP                                                              \\ \cmidrule(lr){2-4}\cmidrule(lr){5-6}
		OurModel-RT         & $\bm{0.299 \pm 0.069}$ & $\bm{0.193 \pm 0.188}$ &  $\bm{0.237 \pm 0.056}$ & $0.412 \pm 0.116$ & $0.290\pm0.076 $       \\
		OurModel-NN2L      & $\bm{0.369 \pm 0.080}$ & $0.395 \pm 0.162$          & $0.396 \pm 0.169$          & $0.390 \pm 0.107$ & $ 4.057\pm0.109 $        \\
		OurModel-NN4L      & $0.413 \pm 0.092$          & $\bm{0.309 \pm 0.063}$ & $0.313 \pm 0.060$          & $0.398 \pm 0.104$ & $ 1.048\pm0.441 $        \\
		OurModel-NN6L      & $0.477 \pm 0.102$          & $0.325 \pm 0.069$          & $\bm{0.297 \pm 0.070}$ & $\bm{0.388 \pm 0.106}$ & $ 0.306\pm0.063 $        \\\cmidrule(lr){2-4}\cmidrule(lr){5-6}
		POTS-FC            & $ 1.477\pm0.220 $          & $ 1.243\pm 0.219$          & $ 1.180\pm0.177 $          & $0.470\pm 0.081$ & $ 0.529\pm0.183 $        \\
		POTS-LSTM            & $ 1.316\pm0.261 $          & $ 1.117\pm0.232 $          & $ 1.179\pm 0.323$          & $0.593\pm0.102 $ & $ 0.613\pm0.137 $        \\
		TARNets $^*$            & $0.780 \pm 0.131$          & $0.570 \pm 0.098$          & $0.701 \pm 0.131$          & $0.568 \pm 0.091$ & $ 0.424\pm0.100 $        \\
		CEVAE $^{**}$            & $1.166 \pm 0.247$          & $1.428 \pm 0.709$          & $0.705 \pm 0.179$          & $\bm{0.337 \pm 0.093}$ & $ \bm{0.232\pm 0.061}$        \\
		FCA            & $ 0.391\pm0.078 $          & $ 1.065\pm0.504 $          & $ 0.682\pm0.080 $          & $ 0.393\pm0.103 $ & $ \bm{0.261\pm0.063} $        \\ \bottomrule
	\end{tabular}
	\begin{tablenotes}\scriptsize
		\item $^*$ We used code from: https://github.com/clinicalml/cfrnet
		\item $^{**}$ We used code from: https://github.com/AMLab-Amsterdam/CEVAE
	\end{tablenotes}

	\vskip-3px
\end{table}
\else
\begin{table*}
	\caption{Out-of-sample ATE error ($\epsilon_\textrm{ATE}$) of each method on different datasets (lower is better). }
	\vskip -9pt
	\label{tab:mae}
	\centering
	\small
	\setlength{\tabcolsep}{4.5pt}
	\begin{tabular}{lccccc}
		
		\toprule
		\multicolumn{1}{c}{\multirow{2}{*}{Method}} & \multicolumn{3}{c}{\begin{tabular}[c]{@{}c@{}}Temporal Data (latent stochastic confounders)\end{tabular}} & \multicolumn{2}{c}{\begin{tabular}[c]{@{}c@{}}IID Data  (\emph{iid} confounders)\end{tabular}} \\ \cmidrule(lr){2-4}\cmidrule(lr){5-6}
		\multicolumn{1}{c}{}                                 & TD2L                   & TD4L                  & TD6L                  & TD6L-iid& IHDP                                                              \\ \cmidrule(lr){2-4}\cmidrule(lr){5-6}
		OurModel-Mat\'ern kernel         & \textbf{0.299$\pm$.069} & \textbf{0.193$\pm$.088} &  \textbf{0.237$\pm$.056} & 0.412$\pm$.116 & \textbf{0.290$\pm$.076}        \\
		OurModel-RBF kernel        & \textbf{0.341$\pm$.078} & \textbf{0.289$\pm$.025} &  \textbf{0.287$\pm$.067} & 0.397$\pm$.096 & 0.460$\pm$.130        \\
		OurModel-RQ kernel        & \textbf{0.311$\pm$.067} & \textbf{0.255$\pm$.020} &  \textbf{0.257$\pm$.063} & \textbf{0.342$\pm$.076} & 0.489$\pm$.170        \\

		\cmidrule(lr){2-4}\cmidrule(lr){5-6}
		OurModel-NN2L      & \textbf{0.369$\pm$.080} & 0.395$\pm$.162          & 0.396$\pm$.169          & 0.390$\pm$.107 &  4.057$\pm$.109         \\
		OurModel-NN4L      & 0.413$\pm$.092          & \textbf{0.309$\pm$.063} & 0.313$\pm$.060          & 0.398$\pm$.104 &  1.048$\pm$.441         \\
		OurModel-NN6L      & 0.477$\pm$.102          & 0.325$\pm$.069          & \textbf{0.297$\pm$.070} & \textbf{0.388$\pm$.106} &  0.306$\pm$.063         \\\cmidrule(lr){2-4}\cmidrule(lr){5-6}
		OurModel-iid confounder-Mat\'ern kernel         & 0.519$\pm$.072 & 0.658$\pm$.098 &  0.762$\pm$.068 & \textbf{0.311$\pm$.092} & \textbf{0.217$\pm$.095}        \\\cmidrule(lr){2-4}\cmidrule(lr){5-6}
		\citet{bojinov2019time} (FC)            &  1.477$\pm$.220           &  1.243$\pm$.219          &  1.180$\pm$.177           & 0.470$\pm$.081 &  0.529$\pm$.183         \\
		\citet{bojinov2019time} (LSTM)            &  1.316$\pm$.261           &  1.117$\pm$.232           &  1.179$\pm$.323          & 0.593$\pm$.102  &  0.613$\pm$.137         \\
		\citet{shalit2017estimating} (TARNets)            & 0.780$\pm$.131          & 0.570$\pm$.098          & 0.701$\pm$.131          & 0.568$\pm$.091 &  0.424$\pm$.100         \\
		\citet{louizos2017causal} (CEVAE)            & 1.166$\pm$.247          & 1.428$\pm$.709          & 0.705$\pm$.179          & \textbf{0.337$\pm$.093} &  \textbf{0.232$\pm$.061}        \\
		\citet{madras2019fairness}            &  0.391$\pm$.078           &  1.065$\pm$.504           &  0.682$\pm$.080           &  0.393$\pm$.103  &  \textbf{0.261$\pm$.063}         \\ \bottomrule
	\end{tabular}

	\vskip-9px
\end{table*}
\fi

\begin{figure}[!ht]
\centering
\includegraphics[width=0.47\textwidth]{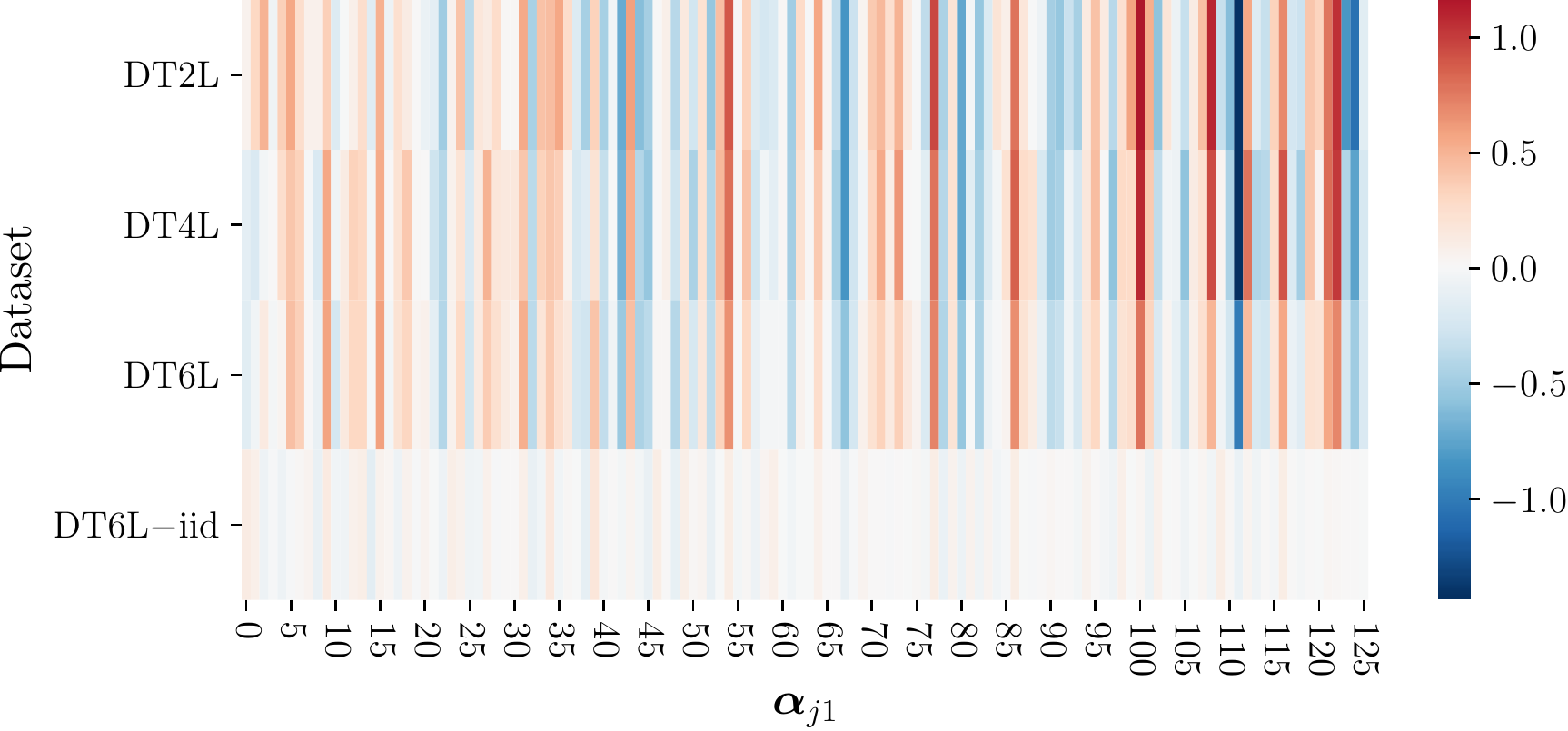}
\vskip -6pt
\caption{The heatmap of $\bm{\alpha}_{j1}$ on each dataset.}
\label{fig:weightz-heatmap}
\vskip -20pt
\end{figure}

\textbf{The results and discussions.} Each of the experimental datasets has 10 replications. For each replication, we use the first 64\% for training, the next 20\% for validation, and the last 16\% for testing. We examine three setups for our inference method, one with kernel method to model the nonlinear functions, another one with the neural networks, and the third one is an iid confounder setting with kernel method to model the nonlinear functions. Specifically, we denote `OurModel-Mat\'ern~kernel', `OurModel-RBF~kernel' and `OurModel-RQ~kernel' as our proposed method with representer theorem to model the nonlinear functions $f_i$ ($i \in \bm{\mathcal{A}}$) using Mat\'ern~kernel, radial basis function kernel, and rational quadratic function kernel, respectively. We denote OurModel-NN$j$L as our framework using neural networks to model these nonlinear functions, where $j \in \{2,4,6\}$ is the number of hidden layers. We also denote `OurModel-iid confounder-Mat\'ern kernel' as our proposed method with representer theorem and independent confounders.

Table~\ref{tab:mae} reports the error of each method, where significant results are highlighted in bold.
We observe that the performance  of  our  model is  competitive for the first three datasets since our framework is \emph{suited} and built for temporal data. This verifies the effectiveness of our proposed framework on the inference of the causal effect for the temporal data, especially with the latent stochastic confounders.
Moreover, the use of representer theorem returns \emph{similar values} of ATE on three different kernel functions. %
Additionally, our proposed method outperforms the other baselines on the first three datasets, this is because we consider the time-dependency in the latent confounders, while the others do not take into account such property. For datasets that respects \emph{iid} confounder, our methods give comparable results with the other baselines (the last five lines in Table~\ref{tab:mae}).
This can be explained as follows. In our setup, $\mathbf{z}_t$ has the following form: $\mathbf{z}_t = \bm{\alpha}_0 + \sum_{j=1}^T\sum_{l=1}^L\bm{\alpha}_{jl}\,k_z(\mathbf{z}_{t-1}, \mathbf{z}_j^l) + \bm{\epsilon}_t$, where $\bm{\alpha}_0$ is a bias vector, $\bm{\alpha}_{jl}$ is a weight vector, and $\bm{\epsilon}_t$ is the Gaussian noise (Please refer to Appendix for specific expressions of $\bm{\alpha}_0$ and $\bm{\alpha}_{jl}$). The learned weights $\bm{\alpha}_{jl}$ presented in Figure~\ref{fig:weightz-heatmap}  
(for $l=1$) 
shows that their quantities are around 0 on data with \textit{iid} confounders, which breaks the connection from $\mathbf{z}_{t-1}$ to $\mathbf{z}_t$ and thus makes these two variables independent to each other.
\begin{figure}
\centering
\includegraphics[width=0.4\textwidth]{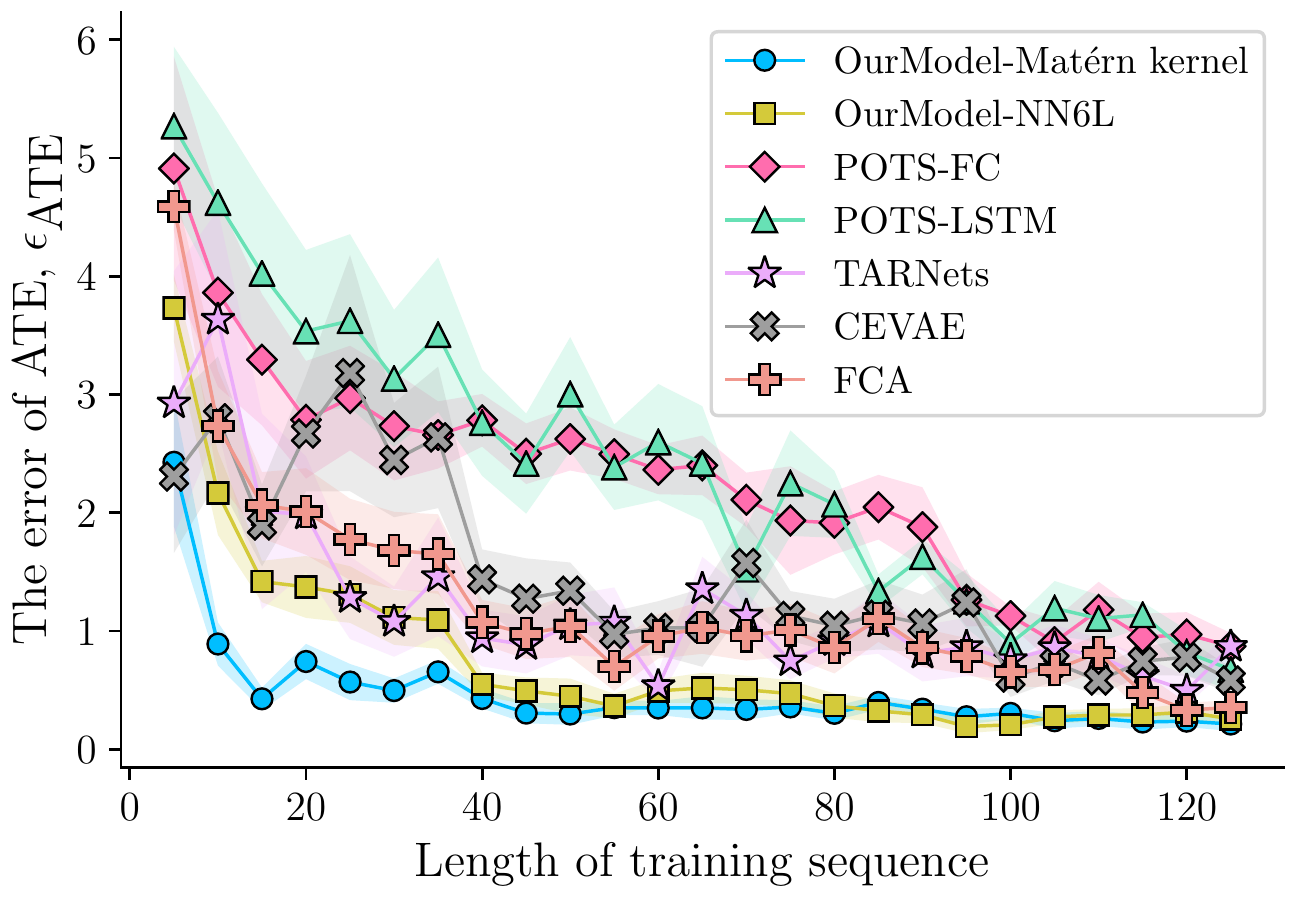}
\vskip -6pt
\caption{$\epsilon_\textrm{ATE}$ on different lengths of the training set.}\label{fig:synthetic-convergence}
\vskip -16pt
\end{figure}
Figure~\ref{fig:synthetic-convergence} presents the convergence of each method on the dataset TD6L, over different lengths of the training set $T_\text{train} \in \{5,10,\!...,125\}$. In general, the more training data we have, the smaller the error of the estimated ATE. The figure reveals that our method (blue line) starts to converge from around $T_{\text{train}} = 45$, which is faster than the others. %
Additionally, the estimated ATE of our method is stable with a small error bar.
\vspace{-0.2cm}

\subsection{Real Data: Gold--Oil Dataset}

Gold is one of the most transactable precious metals, and oil is one of the most transactable commodities. Rising oil price tends to generate higher inflation which strengthens the demand for gold and hence pushes up the gold price \citep{le2011oil,vsimakova2011analysis}. 
In this section, we examine the causal effects from the price of crude oil to that of gold. 
The dataset in this experiment consists of monthly prices of some commodities including gold, crude oil, beef, chicken, cocoa-beans, rice, shrimp, silver, sugar, gasoline, heating oil and natural-gas from May 1989 to May 2019. We consider the price of gold as the outcome $\mathbf{y}$, and the trend of crude oil's price as the treatment $\mathbf{w}$. Specifically, we cast an increase of crude oil's price to 1 ($w_t=1$) and a decrease of crude oil's price to 0 ($w_t=0$). We use the prices of gasoline, heating oil, natural-gas, beef, chicken, cocoa-beans, rice, shrimp, silver and sugar as proxy variables $\mathbf{x}$. %
We evaluate the effect path and ATE between two sequences of treatments $\mathbf{w}_1 = [1,1,\!...,1,1]^\top$ (increasing crude oil prices) and $\mathbf{w}_2 = [0,1,\!...,0,1]^\top$ (alternating decreasing and increasing crude oil prices, i.e., constant on average). In Figure~\ref{fig:ate-gold-oil},~Case~(a) presents the estimated ATE of the above two sequences of treatments.
The estimated ATE using our framework is $4.8$, which means that in a period the price of crude oil increases, the average gold price is about to increase $4.8$. 
This quantity is equivalent to an increase of $0.77\%$ in the gold price over the period ($4.8/\{(\sum_{t=1}^Ty_t^{obs})/T\}= 0.77\%$).
To validate the $0.77\%$ increase in gold price, we contrast and compare this with the results reported in \citet{vsimakova2011analysis} (based on Granger causality) that show the ``percentage increase in oil price leads to a $0.64\%$ increase in gold price''. We note that our results give similar order of magnitude and the slight difference may be attributed to our experimental data that is from May 1989 to May 2019, while the analysis in \citet{vsimakova2011analysis} is on the data from 1970 to 2010. 

\begin{figure}
\centering
    \includegraphics[width=0.43\textwidth]{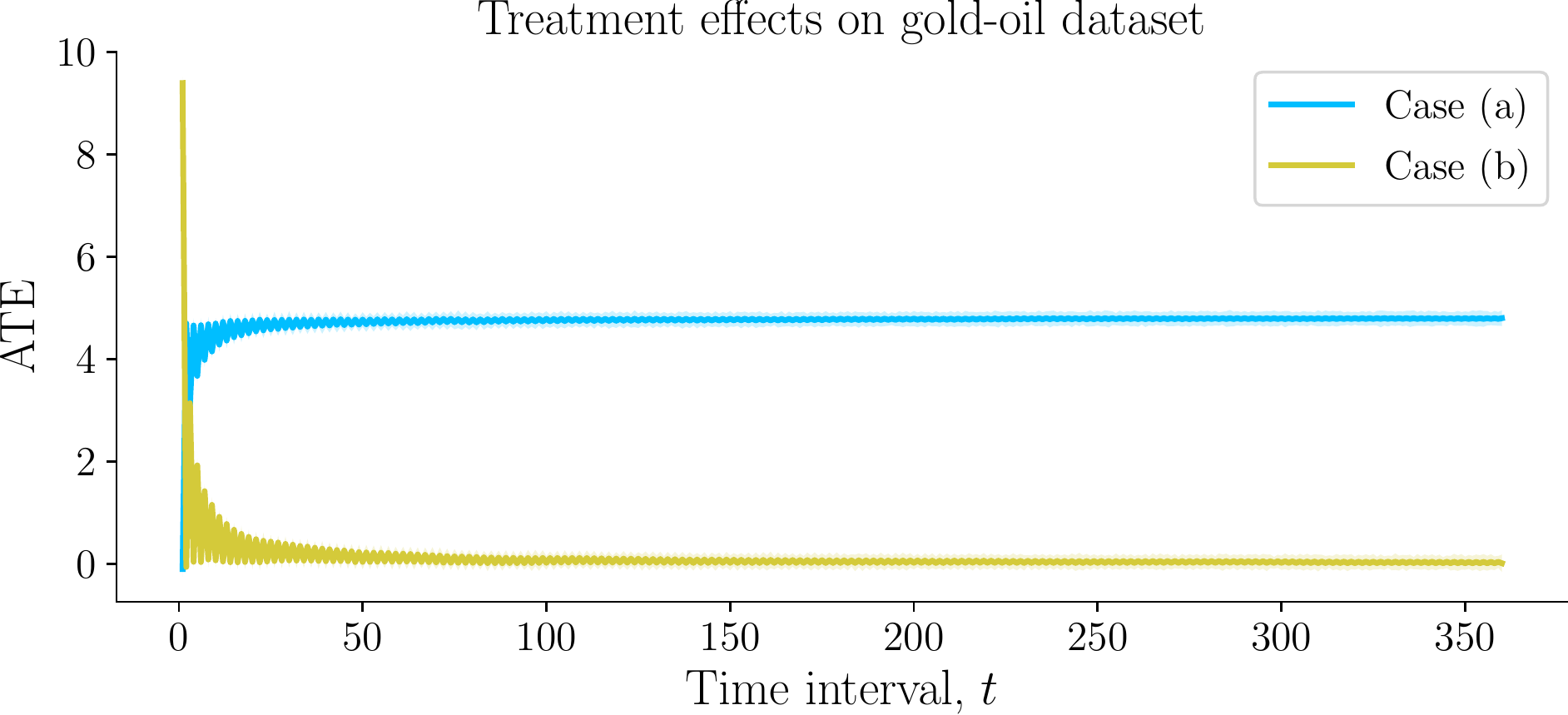}
    \vskip -6pt
    \caption{
        \emph{Case~}(\emph{a}): ATE between two treatment paths $[1,1,\!...,1,1]^\top$ and $[0,1,\!...,0,1]^\top$, \emph{Case~}(\emph{b}): ATE between two treatment paths $[1,0,\!...,1,0]^\top$ and $[0,1,\!...,0,1]^\top$.
    } \label{fig:ate-gold-oil}
    \vskip -12pt
\end{figure}

In Figure~\ref{fig:ate-gold-oil},~Case~(b), we present another experiment with $\mathbf{w}_1 = [1, 0,\!..., 1, 0]^\top$ and $\mathbf{w}_2 = [0, 1,\!..., 0, 1]^\top$. In this case, both treatment paths represent the alternating variation of crude oil prices. Specifically, the former increases first and then decreases, while the latter is on the opposite. The average treatment effect is expected to be around 0.
From Figure~\ref{fig:ate-gold-oil},~Case~(b), the estimated ATE of our method is 0.0045, which is in line with the expectation. 
To check on statistical significance, we performed a one group $t$-test on the EP (Definition~\ref{def:ep}) with the population mean to be tested is 0. The $p$-value given by the $t$-test is 0.9931, which overwhelmingly fail to reject the null hypothesis that the ATE equals 0. This again verifies that our method works sensibly.

\section{CONCLUSION}

We have developed a causal modeling framework that admits confounders as random processes, generalizing recent work where the confounders are assumed to be independent and identically distributed. We study the causal effects over time using variational inference in conjunction with an alternative form of the Representer Theorem with a random input space. 
Our algorithm supports causal inference from the observed outcomes, treatments, and covariates, without parametric specification of the components and their relations. This property is important for capturing real-life causal effects in SCM, where non-linear functions are typically placed in the priors. Our setup admits non-linear functions modulating the SCM with estimated parameters that have analytical solutions. This approach compares favorably to recent techniques that model similar non-linear functions to estimate the causal effects with neural networks, which usually involve extensive model tuning and architecture building. One limitation of our framework is that the fixed amount of passing time (time-lag) is set to unity for our purpose as it leads to further simplifications in the course of computing causal effects. Of practical interest is to perform a more detailed empirical study for general time-lag.%

 \bibliographystyle{apalike}
\bibliography{ref}



\section*{Supplementary material}

# Appendix: Causal Modeling with Stochastic Confounders

Thanh Vinh Vo[1]    Pengfei Wei[1]    Wicher Bergsma[2]    Tze-Yun Leong[1]

[1]School of Computing, National University of Singapore
[2]Department of Statistics, London School of Economics and Political Science

## 1 ADDITIONAL EXPERIMENTS

This section aims to present more experiments on the effectiveness of our proposed kernel method. In particular, we report some additional experiments on our kernel method when the model is *iid* confounders. Specifically, two lines of additional experiments are added in Table 1: 'OurModel-iid confounder-RBF kernel' and 'OurModel-iid confounder-RQ kernel'.

Because the model is *iid* confounders, they can fit the data with *iid* confounders well. Indeed, Table 1 shows that the three models 'OurModel-iid confounder-Matérn kernel', 'OurModel-iid confounder-RBF kernel' and 'OurModel-iid confounder-RQ kernel' slightly outperform the ones with stochastic confounders 'OurModel-Matérn kernel', 'OurModel-RBF kernel' and 'OurModel-RQ kernel' on two *iid* confounders datasets TD6L-iid and IHDP. It also shows that our model in this case (iid confounders and using kernels) achieves competitive results compared to that of Louizos et al. (2017) (CEVAE) and Madras et al. (2019) on TD6L-iid and IHDP.

Table 1: Out-of-sample ATE error ($\epsilon_{\text{ATE}}$) of each method on different datasets (lower is better).

| Method | Temporal Data (latent stochastic confounders) | | | IID Data (*iid* confounders) | |
|---|---|---|---|---|---|
| | TD2L | TD4L | TD6L | TD6L-iid | IHDP |
| OurModel-Matérn kernel | **0.299**±**.069** | **0.193**±**.088** | **0.237**±**.056** | 0.412±.116 | 0.290±.076 |
| OurModel-RBF kernel | **0.341**±**.078** | **0.289**±**.025** | **0.287**±**.067** | 0.397±.096 | 0.460±.130 |
| OurModel-RQ kernel | **0.311**±**.067** | **0.255**±**.020** | **0.257**±**.063** | 0.342±.076 | 0.489±.170 |
| OurModel-NN2L | **0.369**±**.080** | 0.395±.162 | 0.396±.169 | 0.390±.107 | 4.057±.109 |
| OurModel-NN4L | 0.413±.092 | **0.309**±**.063** | 0.313±.060 | 0.398±.104 | 1.048±.441 |
| OurModel-NN6L | 0.477±.102 | 0.325±.069 | **0.297**±**.070** | 0.388±.106 | 0.306±.063 |
| OurModel-iid confounder-Matérn kernel | 0.519±.072 | 0.658±.098 | 0.762±.068 | **0.311**±**.092** | **0.217**±**.095** |
| OurModel-iid confounder-RBF kernel | 0.632±.085 | 0.982±.120 | 0.884±.078 | **0.346**±**.075** | **0.277**±**.091** |
| OurModel-iid confounder-RQ kernel | 0.640±.085 | 1.081±.103 | 0.785±.077 | **0.335**±**.076** | 0.283±.075 |
| Bojinov and Shephard (2019) (FC) | 1.477±.220 | 1.243±.219 | 1.180±.177 | 0.470±.081 | 0.529±.183 |
| Bojinov and Shephard (2019) (LSTM) | 1.316±.261 | 1.117±.232 | 1.179±.323 | 0.593±.102 | 0.613±.137 |
| Shalit et al. (2017) (TARNets) | 0.780±.131 | 0.570±.098 | 0.701±.131 | 0.568±.091 | 0.424±.100 |
| Louizos et al. (2017) (CEVAE) | 1.166±.247 | 1.428±.709 | 0.705±.179 | **0.337**±**.093** | **0.232**±**.061** |
| Madras et al. (2019) | 0.391±.078 | 1.065±.504 | 0.682±.080 | 0.393±.103 | **0.261**±**.063** |

## 2 PREPROCESSING GOLD-OIL DATA: REMOVE TREND & SEASONALITY

This section presents the preprocessing procedure applied to the Gold-Oil dataset. This dataset consists of monthly prices of some commodities including gold, crude oil, beef, chicken, cocoa-beans, rice, shrimp, silver, sugar, gasoline, heating oil and natural-gas from May 1989 to May 2019. To make it suitable for our setting, the following prerocessing procedure is performed. Let

$$\nu_0, \nu_1, \nu_2, \ldots, \nu_T$$

be time series of a commodity. We take the difference between two consecutive observations as input to our model, i.e.,

$$\Delta \boldsymbol{\nu}_t = \boldsymbol{\nu}_t - \boldsymbol{\nu}_{t-1}.$$

This preprocessing is applied to all the prices of commodities, thus the inputs to our model is actually the differenced series:

$$\Delta \boldsymbol{\nu}_1, \Delta \boldsymbol{\nu}_2, \Delta \boldsymbol{\nu}_3, \ldots, \Delta \boldsymbol{\nu}_T.$$

For the treatment $w_t$ $(t = 1, 2, \ldots, T)$, we further cast it to a binary value, i.e.,

$$w_t = \mathbb{1}(\Delta \boldsymbol{\nu}_t > 0),$$

where $\mathbb{1}(\cdot)$ denotes the indicator function.

By using the differenced series, we are removing trends and seasonal effects (see, e.g, Commandeur and Koopman, 2007, Chap. 10; Jinka and Schwartz, 2016, Chap. 4). Hence, it would be more reasonable to assume that the preprocessed data satisfies our proposed causal graph, i.e., the two consecutive differenced values of the observations are independent given the latent confounders, $\Delta \boldsymbol{\nu}_t \perp\!\!\!\perp \Delta \boldsymbol{\nu}_{t-1} \mid \mathbf{z}_t$.

## 3 PROOF OF LEMMA 1

We restate that $f_y \colon \mathcal{Z} \times \mathcal{W} \times \mathcal{S} \mapsto \mathcal{F}_y$, $f_w \colon \mathcal{Z} \times \mathcal{S} \mapsto \mathcal{F}_w$, $f_z \colon \mathcal{Z} \mapsto \mathcal{F}_z$, $f_s \colon \mathcal{S} \mapsto \mathcal{F}_s$, $f_x \colon \mathcal{Z} \times \mathcal{S} \mapsto \mathcal{F}_x$, and $f_q \colon \mathcal{Y} \times \mathcal{W} \times \mathcal{S} \times \mathcal{X} \to \mathcal{F}_z$. In this work, $\mathcal{F}_y = \mathbb{R}$, $\mathcal{F}_w = \mathbb{R}$, $\mathcal{F}_s = \mathbb{R}^{d_s}$, $\mathcal{F}_x = \mathbb{R}^{d_x}$ and $\mathcal{F}_z = \mathbb{R}^{d_z}$.

We also note that $\mathcal{A} = \{y, w, x, s, z, q\}$.

> **Lemma 1.** Let $\kappa_i$ be kernels and $\mathcal{H}_i$ their associated reproducing kernel Hilbert space (RKHS), where $i \in \mathcal{A}$. Let the empirical risk obtained from the negative ELBO be $\widehat{\mathcal{L}}$. Consider minimizing the following objective function
> 
> $$J = \widehat{\mathcal{L}}\left(\bigcup_{i \in \mathcal{A}} f_i\right) + \sum_{i \in \mathcal{A}} \lambda_i \|f_i\|_{\mathcal{H}_i}^2 \qquad (1)$$
> 
> with respect to functions $f_i$ $(i \in \mathcal{A})$, where $\lambda_i \in \mathbb{R}^+$. Then, the minimizer of (1) has the following form $f_i = \sum_{l=1}^{T \times L} \kappa_i(\cdot, \boldsymbol{\nu}_l^i) \boldsymbol{\beta}_l^i$ $(\forall i \in \mathcal{A})$, where $\boldsymbol{\nu}_l^i$ is the $l^{th}$ input to function $f_i$, i.e., it is a subset of the $l^{th}$ tuple of $\mathcal{D}$, and the coefficients $\boldsymbol{\beta}_l^i$ are vectors in the Hilbert space $\mathcal{F}_i$. This minimizer further emits the following solution: $\boldsymbol{\beta}^i = [\boldsymbol{\beta}_1^i, \ldots, \boldsymbol{\beta}_{TL}^i]^\top = \left(\sum_{l=1}^{L} \mathbf{K}_i^{l\top} \mathbf{K}_i^l + \lambda_i L \mathbf{K}_i\right)^{-1} \sum_{l=1}^{L} \mathbf{K}_i^{l\top} \boldsymbol{\psi}^i$ for $i \in \mathcal{A} \setminus \{w, q\}$ and $\boldsymbol{\psi}^y = \mathbf{y}$, $\boldsymbol{\psi}^x = \mathbf{x}$, $\boldsymbol{\psi}^s = \mathbf{s}$, $\boldsymbol{\psi}^z = \mathbf{K}_q \boldsymbol{\beta}^q$.

*Proof.* We divide the proof of Lemma 1 into two parts (1) and (2) as follows.

**(1) The solution form of $f_i$ $(i \in \mathcal{A})$**

We further define $f_x = [f_{x,1}, \ldots, f_{x,d_x}]$ with $f_{x,d} \colon \mathcal{Z} \times \mathcal{S} \mapsto \mathbb{R}$ $(d = 1, \ldots, d_x)$. Similarly, $f_s = [f_{s,1}, \ldots, f_{s,d_s}]$ with $f_{s,d} \colon \mathcal{S} \mapsto \mathbb{R}$ $(d = 1, \ldots, d_s)$, $f_z = [f_{z,1}, \ldots, f_{z,d_z}]$ with $f_{z,d} \colon \mathcal{Z} \mapsto \mathbb{R}$ $(d = 1, \ldots, d_z)$, and $f_q = [f_{q,1}, \ldots, f_{q,d_z}]$ with $f_{q,d} \colon \mathcal{Y} \times \mathcal{W} \times \mathcal{S} \times \mathcal{X} \mapsto \mathbb{R}$ $(d = 1, \ldots, d_z)$.

Consider the subspaces $\mathcal{B}_i \subset \mathcal{H}_i$, $(i \in \mathcal{A})$ defined as follows:

$$\begin{aligned}
\mathcal{B}_y &= \mathrm{span}\{\kappa_y(\cdot, [w_t, \mathbf{s}_t, \mathbf{z}_t^l]) \,:\, t = 1, \ldots, T; l = 1, \ldots, L\}, \\
\mathcal{B}_w &= \mathrm{span}\{\kappa_w(\cdot, [\mathbf{s}_t, \mathbf{z}_t^l]) \,:\, t = 1, \ldots, T; l = 1, \ldots, L\}, \\
\mathcal{B}_x &= \mathrm{span}\{\kappa_x(\cdot, [\mathbf{s}_t, \mathbf{z}_t^l]) \,:\, t = 1, \ldots, T; l = 1, \ldots, L\}, \\
\mathcal{B}_s &= \mathrm{span}\{\kappa_s(\cdot, \mathbf{s}_t) \,:\, t = 1, \ldots, T\}, \\
\mathcal{B}_z &= \mathrm{span}\{\kappa_z(\cdot, \mathbf{z}_t^l) \,:\, t = 1, \ldots, T; l = 1, \ldots, L\}, \\
\mathcal{B}_q &= \mathrm{span}\{\kappa_q(\cdot, [y_t, \mathbf{x}_t, w_t, \mathbf{s}_t]) \,:\, t = 1, \ldots, T\}.
\end{aligned}$$

We project $f_y$, $f_w$, $f_{x,d}$ $(d = 1, \ldots, d_x)$, $f_{s,d}$ $(d = 1, \ldots, d_s)$, $f_{z,d}$ $(d = 1, \ldots, d_z)$, and $f_{q,d}$ $(d = 1, \ldots, d_z)$ onto the subspaces $\mathcal{B}_y$, $\mathcal{B}_w$, $\mathcal{B}_x$, $\mathcal{B}_s$, $\mathcal{B}_z$, $\mathcal{B}_q$, respectively, to obtain $f_y^{\mathrm{sp}}$, $f_w^{\mathrm{sp}}$, $f_{x,d}^{\mathrm{sp}}$, $f_{s,d}^{\mathrm{sp}}$, $f_{z,d}^{\mathrm{sp}}$ and $f_{q,d}^{\mathrm{sp}}$, and also project them onto the perpendicular spaces of the subspaces to obtain $f_y^\perp$, $f_w^\perp$, $f_{x,d}^\perp$, $f_{s,d}^\perp$, $f_{s,d}^\perp$ and $f_{q,d}^\perp$.

Note that $f_{(\cdot)}^{\text{sp}} + f_{(\cdot)}^{\perp} = f_{(\cdot)}$. Thus, $\|f_{(\cdot)}\|_{\mathcal{H}_{(\cdot)}}^2 = \|f_{(\cdot)}^{\text{sp}}\|_{\mathcal{H}_{(\cdot)}}^2 + \|f_{(\cdot)}^{\perp}\|_{\mathcal{H}_{(\cdot)}}^2 \geq \|f_{(\cdot)}^{\text{sp}}\|_{\mathcal{H}_{(\cdot)}}^2$, which implies that $\lambda_{(\cdot)} \|f_{(\cdot)}\|_{\mathcal{H}_{(\cdot)}}^2$ is minimized if $f_{(\cdot)}$ is in its subspace $\mathcal{B}_{(\cdot)}$. $\hspace{2cm} (a)$

Moreover, from the reproducing property, we have that

$$f_y(w_t, \mathbf{s}_t, \mathbf{z}_t^l) = \langle f_y, \kappa_y(\cdot, [w_t, \mathbf{s}_t, \mathbf{z}_t^l])\rangle_{\mathcal{H}_y} = \langle f_y^{\text{sp}} + f_y^{\perp}, \kappa_y(\cdot, [w_t, \mathbf{s}_t, \mathbf{z}_t^l])\rangle_{\mathcal{H}_y}$$
$$= \langle f_y^{\text{sp}}, \kappa_y(\cdot, [w_t, \mathbf{s}_t, \mathbf{z}_t^l])\rangle_{\mathcal{H}_y} + \langle f_y^{\perp}, \kappa_y(\cdot, [w_t, \mathbf{s}_t, \mathbf{z}_t^l])\rangle_{\mathcal{H}_y}$$
$$= \langle f_y^{\text{sp}}, \kappa_y(\cdot, [w_t, \mathbf{s}_t, \mathbf{z}_t^l])\rangle_{\mathcal{H}_y} = f_y^{\text{sp}}(w_t, \mathbf{s}_t, \mathbf{z}_t^l),$$

and similarly $f_w(\mathbf{z}_t^l, \mathbf{s}_t) = f_w^{\text{sp}}(\mathbf{z}_t^l, \mathbf{s}_t)$, $f_{x,d}(\mathbf{z}_t^l, \mathbf{s}_t) = f_{x,d}^{\text{sp}}(\mathbf{z}_t^l, \mathbf{s}_t)$, $f_{s,d}(\mathbf{s}_t) = f_{s,d}^{\text{sp}}(\mathbf{s}_t)$, $f_{z,d}(\mathbf{z}_t) = f_{z,d}^{\text{sp}}(\mathbf{z}_t)$, and $f_{q,d}(y_t, w_t, \mathbf{s}_t, \mathbf{x}_t) = f_{q,d}^{\text{sp}}(y_t, w_t, \mathbf{s}_t, \mathbf{x}_t)$. Hence, we have

$$\widehat{\mathcal{L}}\left(\bigcup_{i \in \mathcal{A}} f_i\right) = \widehat{\mathcal{L}}\left(\bigcup_{i \in \mathcal{A}} f_i^{\text{sp}}\right).$$

The last equation implies that $\widehat{\mathcal{L}}(\cdot)$ depends only on the component of $f_y, f_w, f_{x,d}, f_{s,d}, f_{z,d}, f_{q,d}$ lying in the subspaces $\mathcal{B}_y, \mathcal{B}_w, \mathcal{B}_x, \mathcal{B}_s, \mathcal{B}_z, \mathcal{B}_q$, respectively. $\hspace{2cm} (b)$

From $(a)$ and $(b)$, we have

$$f_y = \sum_{l=1}^{TL} \kappa_y(\cdot, [w_l, \mathbf{s}_l, \mathbf{z}_l]) \boldsymbol{\beta}_l^y, \hspace{2cm} f_w = \sum_{l=1}^{TL} \kappa_w(\cdot, [\mathbf{s}_l, \mathbf{z}_l]) \boldsymbol{\beta}_l^w,$$
$$f_x = \sum_{l=1}^{TL} \kappa_x(\cdot, [\mathbf{s}_l, \mathbf{z}_l]) \boldsymbol{\beta}_l^x, \hspace{2cm} f_s = \sum_{l=1}^{TL} \kappa_s(\cdot, \mathbf{s}_l) \boldsymbol{\beta}_l^s,$$
$$f_z = \sum_{l=1}^{TL} \kappa_z(\cdot, \mathbf{z}_l) \boldsymbol{\beta}_l^z, \hspace{2cm} f_q = \sum_{l=1}^{TL} \kappa_q(\cdot, [y_l, \mathbf{x}_l, w_l, \mathbf{s}_l]) \boldsymbol{\beta}_l^q,$$

where $\boldsymbol{\beta}_l^i$ is in the Hilbert space $\mathcal{F}_i$ ($i \in \mathcal{A}$). This completes the proof on the solution form of $f_i$ ($i \in \mathcal{A}$). We further derive the solution of $\boldsymbol{\beta}_l^i$ ($i \in \mathcal{A}$) in the next subsection.

**(2) The solution of $\boldsymbol{\beta}^i$ ($i \in \mathcal{A}$)**

Let $\mathbf{x}_{:,d}, \mathbf{s}_{:,d}, \mathbf{z}_{:,d}$ be the $d$-th column of $\mathbf{x}, \mathbf{s}$ and $\mathbf{z}$, respectively, i.e., each element of $\mathbf{x}_{:,d}, \mathbf{s}_{:,d}, \mathbf{z}_{:,d}$ is the $d$-th dimension of $\mathbf{x}_t, \mathbf{s}_t$ and $\mathbf{z}_t$.

We further denote $\mathbf{K}_s, \mathbf{K}_q \in \mathbb{R}^{T \times T}$ as kernel matrices computed with the kernel functions $\kappa_s(\mathbf{s}_i, \mathbf{s}_j)$ and $\kappa_q([\mathbf{x}_i, y_i, w_i, \mathbf{s}_i], [\mathbf{x}_j, y_j, w_j, \mathbf{s}_j])$, respectively.

The following form of matrices are applied to $\mathbf{K}_x, \mathbf{K}_w, \mathbf{K}_z, \mathbf{K}_y, \mathbf{K}_x^l, \mathbf{K}_w^l, \mathbf{K}_z^l, \mathbf{K}_y^l$:

$$\mathbf{K}_{(\cdot)} = \begin{bmatrix} \kappa_{(\cdot)}(\Phi_1^1, \Phi_1^1) \ldots \kappa_{(\cdot)}(\Phi_1^1, \Phi_T^1) \ldots \kappa_{(\cdot)}(\Phi_1^1, \Phi_1^L) \ldots \kappa_{(\cdot)}(\Phi_1^1, \Phi_T^L) \\ \vdots \ddots \vdots \vdots \ddots \vdots \\ \kappa_{(\cdot)}(\Phi_T^1, \Phi_1^1) \ldots \kappa_{(\cdot)}(\Phi_T^1, \Phi_T^1) \ldots \kappa_{(\cdot)}(\Phi_T^1, \Phi_1^L) \ldots \kappa_{(\cdot)}(\Phi_T^1, \Phi_T^L) \\ \vdots \vdots \vdots \vdots \\ \kappa_{(\cdot)}(\Phi_1^L, \Phi_1^1) \ldots \kappa_{(\cdot)}(\Phi_1^L, \Phi_T^1) \ldots \kappa_{(\cdot)}(\Phi_1^L, \Phi_1^L) \ldots \kappa_{(\cdot)}(\Phi_1^L, \Phi_T^L) \\ \vdots \ddots \vdots \vdots \ddots \vdots \\ \kappa_{(\cdot)}(\Phi_T^L, \Phi_1^1) \ldots \kappa_{(\cdot)}(\Phi_T^L, \Phi_T^1) \ldots \kappa_{(\cdot)}(\Phi_T^L, \Phi_1^L) \ldots \kappa_{(\cdot)}(\Phi_T^L, \Phi_T^L) \end{bmatrix} \in \mathbb{R}^{TL \times TL},$$

$$\mathbf{K}_{(\cdot)}^l = \begin{bmatrix} \kappa_{(\cdot)}(\Phi_1^l, \Phi_1^1) \ldots \kappa_{(\cdot)}(\Phi_1^l, \Phi_T^1) \ldots \kappa_{(\cdot)}(\Phi_1^l, \Phi_1^L) \ldots \kappa_{(\cdot)}(\Phi_1^l, \Phi_T^L) \\ \vdots \ddots \vdots \vdots \ddots \vdots \\ \kappa_{(\cdot)}(\Phi_T^l, \Phi_1^1) \ldots \kappa_{(\cdot)}(\Phi_T^l, \Phi_T^1) \ldots \kappa_{(\cdot)}(\Phi_T^l, \Phi_1^L) \ldots \kappa_{(\cdot)}(\Phi_T^l, \Phi_T^L) \end{bmatrix} \in \mathbb{R}^{T \times TL},$$

where

- For $\mathbf{K}_x$ and $\mathbf{K}_x^l$, $\kappa_{(\cdot)}(\Phi_i^a, \Phi_j^b) = \kappa_x([\mathbf{s}_i, \mathbf{z}_i^a], [\mathbf{s}_j, \mathbf{z}_j^b])$,

- For $\mathbf{K}_w$ and $\mathbf{K}_w^l$, $\kappa_{(\cdot)}(\Phi_i^a, \Phi_j^b) = \kappa_w([\mathbf{s}_i, \mathbf{z}_i^a], [\mathbf{s}_j, \mathbf{z}_j^b])$,

- For $\mathbf{K}_z$ and $\mathbf{K}_z^l$, $\kappa_{(\cdot)}(\Phi_i^a, \Phi_j^b) = \kappa_z(\mathbf{z}_i^a, \mathbf{z}_j^b)$,

- For $\mathbf{K}_y$ and $\mathbf{K}_y^l$, $\kappa_{(\cdot)}(\Phi_i^a, \Phi_j^b) = \kappa_y([w_i, \mathbf{s}_i, \mathbf{z}_i^a], [w_i, \mathbf{s}_j, \mathbf{z}_j^b])$.

The regularized empirical loss function can be expanded as follows:

$$
\begin{aligned}
J &= \widehat{\mathcal{L}}(f_y, f_w, f_x, f_s, f_z, f_q) + \lambda_y \|f_y\|_{\mathcal{H}_y}^2 + \lambda_w \|f_w\|_{\mathcal{H}_w}^2 + \sum_{d=1}^{d_x} \lambda_x \|f_{x,d}\|_{\mathcal{H}_x}^2 \\
&\quad + \sum_{d=1}^{d_s} \lambda_s \|f_{s,d}\|_{\mathcal{H}_s}^2 + \sum_{d=1}^{d_z} \lambda_z \|f_{z,d}\|_{\mathcal{H}_z}^2 + \sum_{d=1}^{d_z} \lambda_q \|f_{q,d}\|_{\mathcal{H}_q}^2 \\
&= \frac{1}{L} \sum_{l=1}^{L} \Bigg( \frac{1}{2\sigma_y^2} ((\boldsymbol{\beta}^y)^\top {\mathbf{K}_y^l}^\top \mathbf{K}_y^l \boldsymbol{\beta}^y - 2\mathbf{y}^\top \mathbf{K}_y^l \boldsymbol{\beta}^y + \mathbf{y}^\top \mathbf{y}) + T\log\sigma_y \\
&\quad - \mathbf{w}^\top \log\varphi(\mathbf{K}_w^l \boldsymbol{\beta}^w) - (\mathbf{1}-\mathbf{w})^\top \log\varphi(-\mathbf{K}_w^l \boldsymbol{\beta}^w) \\
&\quad + \frac{1}{2\sigma_s^2} \sum_{d=1}^{d_s} ((\boldsymbol{\beta}_d^s)^\top \mathbf{K}_s^\top \mathbf{K}_s \boldsymbol{\beta}_d^s - 2\mathbf{s}_{:,d}^\top \mathbf{K}_s \boldsymbol{\beta}_d^s + \mathbf{s}_{:,d}^\top \mathbf{s}_{:,d}) + T d_s \log\sigma_s \\
&\quad + \frac{1}{2\sigma_x^2} \sum_{d=1}^{d_x} ((\boldsymbol{\beta}_d^x)^\top {\mathbf{K}_x^l}^\top \mathbf{K}_x^l \boldsymbol{\beta}_d^x - 2\mathbf{x}_{:,d}^\top \mathbf{K}_x^l \boldsymbol{\beta}_d^x + \mathbf{x}_{:,d}^\top \mathbf{x}_{:,d}) + T d_x \log\sigma_x \\
&\quad + \frac{1}{2\sigma_z^2} \sum_{d=1}^{d_z} ((\boldsymbol{\beta}_d^z)^\top {\mathbf{K}_z^l}^\top \mathbf{K}_z^l \boldsymbol{\beta}_d^z - 2(\boldsymbol{\beta}_d^q)^\top \mathbf{K}_q^\top \mathbf{K}_z^l \boldsymbol{\beta}_d^z + (\boldsymbol{\beta}_d^q)^\top \mathbf{K}_q^\top \mathbf{K}_q \boldsymbol{\beta}_d^q) + T d_z \log(\sigma_z/\sigma_q) + \frac{1}{2} T d_z \sigma_q^2/\sigma_z^2 \Bigg) \\
&\quad + \lambda_y (\boldsymbol{\beta}^y)^\top \mathbf{K}_y \boldsymbol{\beta}^y + \lambda_w (\boldsymbol{\beta}^w)^\top \mathbf{K}_w \boldsymbol{\beta}^w + \sum_{d=1}^{d_x} \lambda_x (\boldsymbol{\beta}_d^x)^\top \mathbf{K}_x \boldsymbol{\beta}_d^x \\
&\quad + \sum_{d=1}^{d_s} \lambda_s (\boldsymbol{\beta}_d^s)^\top \mathbf{K}_s \boldsymbol{\beta}_d^s + \sum_{d=1}^{d_z} \lambda_z (\boldsymbol{\beta}_d^z)^\top \mathbf{K}_z \boldsymbol{\beta}_d^z + \sum_{d=1}^{d_z} \lambda_q (\boldsymbol{\beta}_d^q)^\top \mathbf{K}_q \boldsymbol{\beta}_d^q
\end{aligned}
$$

Taking derivative of $J$ with respect to $\boldsymbol{\beta}^y$ and equates to 0, we obtain

$$
\nabla_{\boldsymbol{\beta}^y} J = \frac{1}{L} \sum_{l=1}^{L} \left( \frac{1}{2\sigma_y^2} (2{\mathbf{K}_y^l}^\top \mathbf{K}_y^l \boldsymbol{\beta}^y - 2{\mathbf{K}_y^l}^\top \mathbf{y}) \right) + 2\lambda_y \mathbf{K}_y \boldsymbol{\beta}^y = 0
$$

$$
\Longrightarrow \boldsymbol{\beta}^y = \left( \frac{1}{\sigma_y^2} \frac{1}{L} \sum_{l=1}^{L} {\mathbf{K}_y^l}^\top \mathbf{K}_y^l + 2\lambda_y \mathbf{K}_y \right)^{-1} \frac{1}{\sigma_y^2} \left( \frac{1}{L} \sum_{l=1}^{L} \mathbf{K}_y^l \right)^\top \mathbf{y}
$$

$$
= \left( \sum_{l=1}^{L} {\mathbf{K}_y^l}^\top \mathbf{K}_y^l + 2\lambda_y L \sigma_y^2 \mathbf{K}_y \right)^{-1} \left( \sum_{l=1}^{L} \mathbf{K}_y^l \right)^\top \mathbf{y}.
$$

Similarly, taking derivative of $J$ w.r.t $\boldsymbol{\beta}_d^x, \boldsymbol{\beta}_d^s, \boldsymbol{\beta}_d^z$ and set to 0, we obtain the following update rules

$$
\boldsymbol{\beta}_d^x = \left( \sum_{l=1}^{L} {\mathbf{K}_x^l}^\top \mathbf{K}_x^l + 2\lambda_x L \sigma_x^2 \mathbf{K}_x \right)^{-1} \left( \sum_{l=1}^{L} \mathbf{K}_x^l \right)^\top \mathbf{x}_{:,d},
$$

$$
\boldsymbol{\beta}_d^s = \left( \mathbf{K}_s + 2\lambda_s \sigma_s^2 \mathbf{I} \right)^{-1} \mathbf{s}_{:,d},
$$

$$
\boldsymbol{\beta}_d^z = \left( \sum_{l=1}^{L} {\mathbf{K}_z^l}^\top \mathbf{K}_z^l + 2\lambda_z L \sigma_z^2 \mathbf{K}_z \right)^{-1} \sum_{l=1}^{L} \left( {\mathbf{K}_z^l}^\top \mathbf{K}_q \boldsymbol{\beta}_d^q \right).
$$

Absorbing the all the constants into $\lambda_{(.)}$ and denoting $\boldsymbol{\beta}^i = [\boldsymbol{\beta}^i_1, \ldots, \boldsymbol{\beta}^i_{d_i}]$ ($i \in \{x, s, z\}$), we obtain

$$\boldsymbol{\beta}^y = \left(\sum_{l=1}^L {\mathbf{K}^l_y}^\top \mathbf{K}^l_y + \lambda_y L \mathbf{K}_y\right)^{-1} \left(\sum_{l=1}^L \mathbf{K}^l_y\right)^\top \mathbf{y},$$

$$\boldsymbol{\beta}^x = \left(\sum_{l=1}^L {\mathbf{K}^l_x}^\top \mathbf{K}^l_x + \lambda_x L \mathbf{K}_x\right)^{-1} \left(\sum_{l=1}^L \mathbf{K}^l_x\right)^\top \mathbf{x},$$

$$\boldsymbol{\beta}^s = (\mathbf{K}_s + \lambda_s \mathbf{I})^{-1} \mathbf{s},$$

$$\boldsymbol{\beta}^z = \left(\sum_{l=1}^L {\mathbf{K}^l_z}^\top \mathbf{K}^l_z + \lambda_z L \mathbf{K}_z\right)^{-1} \sum_{l=1}^L \left({\mathbf{K}^l_z}^\top \mathbf{K}_q \boldsymbol{\beta}_q\right).$$

We remark that $\boldsymbol{\beta}^s = (\mathbf{K}_s + \lambda_s \mathbf{I})^{-1} \mathbf{s} = \left(\sum_{l=1}^L \mathbf{K}_s \mathbf{K}_s + \lambda_s L \mathbf{K}_s\right)^{-1} (\sum_{l=1}^L \mathbf{K}_s)^\top \mathbf{s}$, thus we can present a general form of the solution as in Lemma 1. □

## 4 PROOF OF LEMMA 2

We repeat Lemma 2 here for convenience:

**Lemma 2.** *For any fixed $\boldsymbol{\beta}^q$, the objective function $J$ in Eq. (1) is convex with respect to $\boldsymbol{\beta}^i$ for all $i \in \mathcal{A} \setminus \{q\}$.*

*Proof.* From Lemma 1, we see that the objective function $J$ is a combination of several components including $(\boldsymbol{\beta}^i)^\top \mathbf{C} \boldsymbol{\beta}^i$, $\mathbf{c}^\top \boldsymbol{\beta}^i$, $-\mathbf{w}^\top \log \varphi(\mathbf{K}^l_w \boldsymbol{\beta}^w)$ and $-(\mathbf{1}-\mathbf{w})^\top \log \varphi(-\mathbf{K}^l_w \boldsymbol{\beta}^w)$, where $i \in \{y, s, x, z\}$, $\mathbf{C}$ is a positive semi-definite matrix and $\mathbf{c}$ is a vector.

For the first and second term, we have

$$\nabla^2_{\boldsymbol{\beta}^i}\left\{(\boldsymbol{\beta}^i)^\top \mathbf{C} \boldsymbol{\beta}^i\right\} = \mathbf{C} + \mathbf{C}^\top = 2\mathbf{C} \succeq 0,$$

$$\nabla^2_{\boldsymbol{\beta}^i}\left\{\mathbf{c}^\top \boldsymbol{\beta}^i\right\} = \mathbf{0} \succeq 0.$$

For the third term, we have

$$\nabla_{\boldsymbol{\beta}^w}\left\{-\mathbf{w}^\top \log \varphi(\mathbf{K}^l_w \boldsymbol{\beta}^w)\right\} = -(\nabla_{\boldsymbol{\beta}^w} \log \varphi(\mathbf{K}^l_w \boldsymbol{\beta}^w))\mathbf{w}$$
$$= -(\mathbf{K}^l_w)^\top \mathrm{diag}(\mathbf{w}) \varphi(-\mathbf{K}^l_w \boldsymbol{\beta}^w),$$

and thus

$$\nabla^2_{\boldsymbol{\beta}^w}\left\{-\mathbf{w}^\top \log \varphi(\mathbf{K}^l_w \boldsymbol{\beta}^w)\right\} = -(\nabla_{\boldsymbol{\beta}^w} \varphi(-\mathbf{K}^l_w \boldsymbol{\beta}^w))((\mathbf{K}^l_w)^\top \mathrm{diag}(\mathbf{w}))^\top$$
$$= -(\nabla_{\boldsymbol{\beta}^w} \varphi(-\mathbf{K}^l_w \boldsymbol{\beta}^w)) \mathrm{diag}(\mathbf{w}) \mathbf{K}^l_w$$
$$= (\mathbf{K}^l_w)^\top \mathrm{diag}(\varphi(-\mathbf{K}^l_w \boldsymbol{\beta}^w) \odot \varphi(\mathbf{K}^l_w \boldsymbol{\beta}^w) \odot \mathbf{w}) \mathbf{K}^l_w \succeq 0.$$

Similarly, for the last term, we have

$$\nabla_{\boldsymbol{\beta}^w}\left\{-(\mathbf{1}-\mathbf{w})^\top \log \varphi(-\mathbf{K}^l_w \boldsymbol{\beta}^w)\right\} = (\mathbf{K}^l_w)^\top \mathrm{diag}(\mathbf{1}-\mathbf{w}) \varphi(\mathbf{K}^l_w \boldsymbol{\beta}^w),$$

and

$$\nabla^2_{\boldsymbol{\beta}^w}\left\{-(\mathbf{1}-\mathbf{w})^\top \log \varphi(-\mathbf{K}^l_w \boldsymbol{\beta}^w)\right\} = (\mathbf{K}^l_w)^\top \mathrm{diag}(\varphi(\mathbf{K}^l_w \boldsymbol{\beta}^w) \odot \varphi(-\mathbf{K}^l_w \boldsymbol{\beta}^w) \odot (\mathbf{1}-\mathbf{w})) \mathbf{K}^l_w \succeq 0.$$

Moreover, the second-order derivative of $J$ with respect to $\sigma_{(.)}$s are also positive. Consequently, $J$ is convex because it is a linear combination of convex functions. □

# 5 AUXILIARY DISTRIBUTION $\widetilde{p}(\mathbf{y} \,|\, \mathbf{s}, \mathbf{x}, \mathbf{w})$, $\widetilde{p}(\mathbf{w} \,|\, \mathbf{s}, \mathbf{x})$ AND $\widetilde{p}(\mathbf{s} \,|\, \mathbf{x})$

This section outlines the approximation of $p(\mathbf{y}|\mathbf{s},\mathbf{x},\mathbf{w})$, $p(\mathbf{w}|\mathbf{s},\mathbf{x})$ and $p(\mathbf{s}|\mathbf{x})$. Denoting their corresponding approximation as $\widetilde{p}(\mathbf{y}|\mathbf{s},\mathbf{x},\mathbf{w})$, $\widetilde{p}(\mathbf{w}|\mathbf{s},\mathbf{x})$, $\widetilde{p}(\mathbf{s}|\mathbf{x})$, we estimate parameters of those distribution using classical representer theorem. Learning parameters of $\widetilde{p}(\mathbf{w} \,|\, \mathbf{s}, \mathbf{x})$ is presented in the main text. In the following, we present $\widetilde{p}(\mathbf{y} \,|\, \mathbf{s}, \mathbf{x}, \mathbf{w})$ and $\widetilde{p}(\mathbf{s} \,|\, \mathbf{x})$.

## 5.1 Learning $\widetilde{p}(\mathbf{y} \,|\, \mathbf{s}, \mathbf{x}, \mathbf{w})$

$$\widetilde{p}(\mathbf{y} \,|\, \mathbf{s}, \mathbf{x}, \mathbf{w}) = \prod_{t=1}^{T} \widetilde{p}(y_t \,|\, y_{t-1}, \mathbf{s}_t, \mathbf{x}_t, w_t) = \prod_{t=1}^{T} \mathsf{N}(y_t \,|\, g_y(y_{t-1}, \mathbf{s}_t, \mathbf{x}_t, w_t), \omega_y^2).$$

The regularized empirical risk is as follows:

$$\begin{aligned} J_y &= \frac{1}{2\omega_y^2} \sum_{t=1}^{T} (y_t - g_y(y_{t-1}, \mathbf{s}_t, \mathbf{x}_t, w_t))^2 + T \log \omega_y + \delta_y \|g_y\|_{\mathcal{V}_y}^2 \\ &= \frac{1}{2\omega_y^2} (\mathbf{y} - \mathbf{K}_{\widetilde{y}} \boldsymbol{\alpha}^y)^\top (\mathbf{y} - \mathbf{K}_{\widetilde{y}} \boldsymbol{\alpha}^y) + T \log \omega_y + \delta_y (\boldsymbol{\alpha}^y)^\top \mathbf{K}_{\widetilde{y}} \boldsymbol{\alpha}^y, \end{aligned}$$

where $\boldsymbol{\alpha}^y = [\alpha_1^y, \ldots, \alpha_T^y]^\top$ is the parameter to be learned and $\mathbf{K}_{\widetilde{y}}$ is the matrix computed with kernel function $\tau_y([y_{i-1}, \mathbf{s}_i, \mathbf{x}_i, w_i], [y_{j-1}, \mathbf{s}_j, \mathbf{x}_j, w_j])$. Taking derivative of $J_y$ w.r.t $\boldsymbol{\alpha}^y$ and set to 0, we obtain

$$\boxed{\boldsymbol{\alpha}^y = (\mathbf{K}_{\widetilde{y}} + 2\delta_y \omega_y^2 \mathbf{I}_T)^{-1} \mathbf{y}.}$$

## 5.2 Learning $\widetilde{p}(\mathbf{s} \,|\, \mathbf{x})$

$$\widetilde{p}(\mathbf{s} \,|\, \mathbf{x}) = \prod_{t=1}^{T} \widetilde{p}(\mathbf{s}_t \,|\, \mathbf{s}_{t-1}, \mathbf{x}_t) = \prod_{t=1}^{T} \mathsf{N}(\mathbf{s}_t \,|\, g_s(\mathbf{s}_{t-1}, \mathbf{x}_t), \omega_s^2 \mathbf{I}_{d_s}).$$

The regularized empirical risk is as follows:

$$\begin{aligned} J_s &= \frac{1}{2\omega_s^2} \sum_{t=1}^{T} \|\mathbf{s}_t - g_s(\mathbf{s}_{t-1}, \mathbf{x}_t)\|_2^2 + T d_s \log \omega_s + \sum_{d=1}^{d_s} \delta_s \|g_{s,d}\|_{\mathcal{V}_s}^2 \\ &= \frac{1}{2\omega_s^2} \sum_{d=1}^{d_s} (\mathbf{s}_{:,d} - \mathbf{K}_{\widetilde{s}} \boldsymbol{\alpha}_d^s)^\top (\mathbf{s}_{:,d} - \mathbf{K}_{\widetilde{s}} \boldsymbol{\alpha}_d^s) + T d_s \log \omega_s + \sum_{d=1}^{d_s} \delta_s (\boldsymbol{\alpha}_d^s)^\top \mathbf{K}_{\widetilde{s}} \boldsymbol{\alpha}_d^s, \end{aligned}$$

where $\mathbf{s}_{:,d}$ denotes the $d$-th column of $\mathbf{s}$, $\boldsymbol{\alpha}_d^s = [\alpha_{1,d}^s, \ldots, \alpha_{T,d}^s]^\top$ is the parameter to be learned and $\mathbf{K}_{\widetilde{s}}$ is the matrix computed with kernel function $\tau_s([\mathbf{s}_{i-1}, \mathbf{x}_i], [\mathbf{s}_{j-1}, \mathbf{x}_j])$. Taking derivative of $J_s$ w.r.t $\boldsymbol{\alpha}_d^s$ and set to 0, we obtain $\boldsymbol{\alpha}_d^s = (\mathbf{K}_{\widetilde{s}} + 2\delta_s \omega_s^2 \mathbf{I}_T)^{-1} \mathbf{s}_{:,d}$ where $d = 1, 2, \ldots, d_s$. Thus we have

$$\boxed{\boldsymbol{\alpha}^s = (\mathbf{K}_{\widetilde{s}} + 2\delta_s \omega_s^2 \mathbf{I}_T)^{-1} \mathbf{s}.}$$

# 6 USING ADDITIVE KERNELS

This section presents the update rules for some additional parameters when using an additive kernel function. For example, we define an additive kernel function as follows

$$\kappa_y([w, s, z], [w', s', z']) = \gamma_0^y + \gamma_w^y \phi_w(w, w') + \gamma_s^y \phi_s(s, s') + \gamma_z^y \phi_z(z, z')$$

where each function $\phi_w(\cdot, \cdot)$, $\phi_s(\cdot, \cdot)$, $\phi_z(\cdot, \cdot)$ is a typical kernel function such as Matérn kernel, RBF kernel, or RQ kernel. Then, the kernel matrix is as follows

$$\mathbf{K}_y = \gamma_0^y \mathbf{J}_{TL} + \gamma_w^y \mathbf{M}_w + \gamma_s^y \mathbf{M}_s + \gamma_z^y \mathbf{M}_z, \qquad \mathbf{K}_y^l = \gamma_0^y \mathbf{J}_{T \times TL} + \gamma_w^y \mathbf{M}_w^l + \gamma_s^y \mathbf{M}_s^l + \gamma_z^y \mathbf{M}_z^l.$$

where $\mathbf{J}_a \in \mathbb{R}^{a \times a}$ and $\mathbf{J}_{a \times b} \in \mathbb{R}^{a \times b}$ are matrices of ones, $\mathbf{M}_w = \mathbf{J}_L \otimes \mathbf{\Lambda}_w$ and $\mathbf{M}_w^l = \mathbf{1}_L^\top \otimes \mathbf{\Lambda}_w$ with $\mathbf{\Lambda}_w \in \mathbb{R}^{T \times T}$ is a kernel matrix computed with kernel function $\phi_w(w_i, w_j)$ and $\mathbf{1}_d$ denotes the $d$-dimensional column vector of ones, $\otimes$ is the Kronecker product; similarly, $\mathbf{M}_s = \mathbf{J}_L \otimes \mathbf{\Lambda}_s$ and $\mathbf{M}_s^l = \mathbf{1}_L^\top \otimes \mathbf{\Lambda}_s$ with $\mathbf{\Lambda}_s \in \mathbb{R}^{T \times T}$ is a kernel matrix computed with kernel function $\phi_s(\mathbf{s}_i, \mathbf{s}_j)$. We remark that $\mathbf{M}_w, \mathbf{M}_s \in \mathbb{R}^{TL \times TL}$ and $\mathbf{M}_w^l, \mathbf{M}_s^l \in \mathbb{R}^{T \times TL}$. Lastly,

$$\mathbf{M}_z = \begin{bmatrix} \phi_z(\mathbf{z}_1^1, \mathbf{z}_1^1) \ldots \phi_z(\mathbf{z}_1^1, \mathbf{z}_T^1) \ldots \phi_z(\mathbf{z}_1^1, \mathbf{z}_1^L) \ldots \phi_z(\mathbf{z}_1^1, \mathbf{z}_T^L) \\ \vdots \ddots \vdots \vdots \ddots \vdots \\ \phi_z(\mathbf{z}_T^1, \mathbf{z}_1^1) \ldots \phi_z(\mathbf{z}_T^1, \mathbf{z}_T^1) \ldots \phi_z(\mathbf{z}_T^1, \mathbf{z}_1^L) \ldots \phi_z(\mathbf{z}_T^1, \mathbf{z}_T^L) \\ \vdots \vdots \vdots \vdots \\ \phi_z(\mathbf{z}_1^L, \mathbf{z}_1^1) \ldots \phi_z(\mathbf{z}_1^L, \mathbf{z}_T^1) \ldots \phi_z(\mathbf{z}_1^L, \mathbf{z}_1^L) \ldots \phi_z(\mathbf{z}_1^L, \mathbf{z}_T^L) \\ \vdots \ddots \vdots \vdots \ddots \vdots \\ \phi_z(\mathbf{z}_T^L, \mathbf{z}_1^1) \ldots \phi_z(\mathbf{z}_T^L, \mathbf{z}_T^1) \ldots \phi_z(\mathbf{z}_T^L, \mathbf{z}_1^L) \ldots \phi_z(\mathbf{z}_T^L, \mathbf{z}_T^L) \end{bmatrix} \in \mathbb{R}^{TL \times TL},$$

$$\mathbf{M}_z^l = \begin{bmatrix} \phi_z(\mathbf{z}_1^l, \mathbf{z}_1^1) \ldots \phi_z(\mathbf{z}_1^l, \mathbf{z}_T^1) \ldots \phi_z(\mathbf{z}_1^l, \mathbf{z}_1^L) \ldots \phi_z(\mathbf{z}_1^l, \mathbf{z}_T^L) \\ \vdots \ddots \vdots \vdots \ddots \vdots \\ \phi_z(\mathbf{z}_T^l, \mathbf{z}_1^1) \ldots \phi_z(\mathbf{z}_T^l, \mathbf{z}_T^1) \ldots \phi_z(\mathbf{z}_T^l, \mathbf{z}_1^L) \ldots \phi_z(\mathbf{z}_T^l, \mathbf{z}_T^L) \end{bmatrix} \in \mathbb{R}^{T \times TL}.$$

Similarly, we have

$$\mathbf{K}_w = \gamma_0^w \mathbf{J}_{TL} + \gamma_s^w \mathbf{M}_s + \gamma_z^w \mathbf{M}_z, \qquad \mathbf{K}_w^l = \gamma_0^w \mathbf{J}_{T \times TL} + \gamma_s^w \mathbf{M}_s^l + \gamma_z^w \mathbf{M}_z^l$$
$$\mathbf{K}_x = \gamma_0^x \mathbf{J}_{TL} + \gamma_s^x \mathbf{M}_s + \gamma_z^x \mathbf{M}_z, \qquad \mathbf{K}_x^l = \gamma_0^x \mathbf{J}_{T \times TL} + \gamma_s^x \mathbf{M}_s^l + \gamma_z^x \mathbf{M}_z^l$$
$$\mathbf{K}_z = \gamma_0^z \mathbf{J}_{TL} + \gamma_z^z \mathbf{M}_z, \qquad \mathbf{K}_z^l = \gamma_0^z \mathbf{J}_{T \times TL} + \gamma_z^z \mathbf{M}_z^l$$
$$\mathbf{K}_s = \gamma_0^s \mathbf{J}_T + \gamma_s^s \mathbf{\Lambda}_s, \qquad \mathbf{K}_q = \gamma_0^q \mathbf{J}_T + \gamma_y^q \mathbf{\Lambda}_y + \gamma_w^q \mathbf{\Lambda}_w + \gamma_s^q \mathbf{\Lambda}_s + \gamma_x^q \mathbf{\Lambda}_x.$$

where $\mathbf{M}_s$, $\mathbf{M}_z$, $\mathbf{M}_w$, $\mathbf{M}_s^l$, $\mathbf{M}_z^l$, $\mathbf{M}_w^l$, $\mathbf{\Lambda}_w$, $\mathbf{\Lambda}_s$ are defined as above, and $\mathbf{\Lambda}_x, \mathbf{\Lambda}_y \in \mathbb{R}^{T \times T}$ are kernel matrices computed with kernel functions $\phi_x(\mathbf{x}_i, \mathbf{x}_j)$ and $\phi_y(y_i, y_j)$, respectively.

Substitute the above kernel matrices to the objective function in Section 3, we obtain

$$\begin{aligned} J = \frac{1}{L} \sum_{l=1}^{L} & \bigg( \frac{1}{2\sigma_y^2} ((\boldsymbol{\beta}^y)^\top [\gamma_0^y \mathbf{J}_{T \times TL} + \gamma_w^y \mathbf{M}_w^l + \gamma_s^y \mathbf{M}_s^l + \gamma_z^y \mathbf{M}_z^l]^\top [\gamma_0^y \mathbf{J}_{T \times TL} + \gamma_w^y \mathbf{M}_w^l \\ & + \gamma_s^y \mathbf{M}_s^l + \gamma_z^y \mathbf{M}_z^l] \boldsymbol{\beta}^y - 2 \mathbf{y}^\top [\gamma_0^y \mathbf{J}_{T \times TL} + \gamma_w^y \mathbf{M}_w^l + \gamma_s^y \mathbf{M}_s^l + \gamma_z^y \mathbf{M}_z^l] \boldsymbol{\beta}^y + \mathbf{y}^\top \mathbf{y}) + T \log \sigma_y \\ & - \mathbf{w}^\top \log \varphi([\gamma_0^w \mathbf{J}_{T \times TL} + \gamma_s^w \mathbf{M}_s^l + \gamma_z^w \mathbf{M}_z^l] \boldsymbol{\beta}^w) \\ & - (\mathbf{1} - \mathbf{w})^\top \log \varphi(-[\gamma_0^w \mathbf{J}_{T \times TL} + \gamma_s^w \mathbf{M}_s^l + \gamma_z^w \mathbf{M}_z^l] \boldsymbol{\beta}^w) \\ & + \frac{1}{2\sigma_s^2} \sum_{d=1}^{d_s} ((\boldsymbol{\beta}_d^s)^\top [\gamma_0^s \mathbf{J}_{T \times TL} + \gamma_s^s \mathbf{\Lambda}_s]^\top [\gamma_0^s \mathbf{J}_{T \times TL} + \gamma_s^s \mathbf{\Lambda}_s] \boldsymbol{\beta}_d^s \\ & - 2 \mathbf{s}_{:,d}^\top [\gamma_0^s \mathbf{J}_{T \times TL} + \gamma_s^s \mathbf{\Lambda}_s] \boldsymbol{\beta}_d^s + \mathbf{s}_{:,d}^\top \mathbf{s}_{:,d}) + T d_s \log \sigma_s \\ & + \frac{1}{2\sigma_x^2} \sum_{d=1}^{d_x} ((\boldsymbol{\beta}_d^x)^\top [\gamma_0^x \mathbf{J}_{T \times TL} + \gamma_s^x \mathbf{M}_s^l + \gamma_z^x \mathbf{M}_z^l]^\top [\gamma_0^x \mathbf{J}_{T \times TL} + \gamma_s^x \mathbf{M}_s^l + \gamma_z^x \mathbf{M}_z^l] \boldsymbol{\beta}_d^x \\ & - 2 \mathbf{x}_{:,d}^\top [\gamma_0^x \mathbf{J}_{T \times TL} + \gamma_s^x \mathbf{M}_s^l + \gamma_z^x \mathbf{M}_z^l] \boldsymbol{\beta}_d^x + \mathbf{x}_{:,d}^\top \mathbf{x}_{:,d}) + T d_x \log \sigma_x \\ & + \frac{1}{2\sigma_z^2} \sum_{d=1}^{d_z} ((\boldsymbol{\beta}_d^z)^\top [\gamma_0^z \mathbf{J}_{T \times TL} + \gamma_z^z \mathbf{M}_z^l]^\top [\gamma_0^z \mathbf{J}_{T \times TL} + \gamma_z^z \mathbf{M}_z^l] \boldsymbol{\beta}_d^z \\ & - 2 (\boldsymbol{\beta}_d^q)^\top [\gamma_0^q \mathbf{J}_T + \gamma_y^q \mathbf{\Lambda}_y + \gamma_w^q \mathbf{\Lambda}_w + \gamma_s^q \mathbf{\Lambda}_s + \gamma_x^q \mathbf{\Lambda}_x]^\top [\gamma_0^z \mathbf{J}_{T \times TL} + \gamma_z^z \mathbf{M}_z^l] \boldsymbol{\beta}_d^z \\ & + (\boldsymbol{\beta}_d^q)^\top [\gamma_0^q \mathbf{J}_T + \gamma_y^q \mathbf{\Lambda}_y + \gamma_w^q \mathbf{\Lambda}_w + \gamma_s^q \mathbf{\Lambda}_s + \gamma_x^q \mathbf{\Lambda}_x]^\top [\gamma_0^q \mathbf{J}_T + \gamma_y^q \mathbf{\Lambda}_y \\ & + \gamma_w^q \mathbf{\Lambda}_w + \gamma_s^q \mathbf{\Lambda}_s + \gamma_x^q \mathbf{\Lambda}_x] \boldsymbol{\beta}_d^q) + T d_z \log(\sigma_z / \sigma_q) + \frac{1}{2} T d_z \sigma_q^2 / \sigma_z^2 \bigg) \end{aligned}$$

$$+ \lambda_y(\boldsymbol{\beta}^y)^\top[\gamma_0^y \mathbf{J}_{TL} + \gamma_w^y \mathbf{M}_w + \gamma_s^y \mathbf{M}_s + \gamma_z^y \mathbf{M}_z]\boldsymbol{\beta}^y + \lambda_w(\boldsymbol{\beta}^w)^\top[\gamma_0^w \mathbf{J}_{TL} + \gamma_s^w \mathbf{M}_s + \gamma_z^w \mathbf{M}_z]\boldsymbol{\beta}^w$$

$$+ \sum_{d=1}^{d_x} \lambda_x(\boldsymbol{\beta}_d^x)^\top[\gamma_0^x \mathbf{J}_{TL} + \gamma_s^x \mathbf{M}_s + \gamma_z^x \mathbf{M}_z]\boldsymbol{\beta}_d^x + \sum_{d=1}^{d_s} \lambda_s(\boldsymbol{\beta}_d^s)^\top[\gamma_0^s \mathbf{J}_T + \gamma_s^s \boldsymbol{\Lambda}_s]\boldsymbol{\beta}_d^s$$

$$+ \sum_{d=1}^{d_z} \lambda_z(\boldsymbol{\beta}_d^z)^\top[\gamma_0^z \mathbf{J}_{TL} + \gamma_z^z \mathbf{M}_z]\boldsymbol{\beta}_d^z + \sum_{d=1}^{d_z} \lambda_q(\boldsymbol{\beta}_d^q)^\top[\gamma_0^q \mathbf{J}_T + \gamma_y^q \boldsymbol{\Lambda}_y + \gamma_w^q \boldsymbol{\Lambda}_w + \gamma_s^q \boldsymbol{\Lambda}_s + \gamma_x^q \boldsymbol{\Lambda}_x]\boldsymbol{\beta}_d^q.$$

Taking derivative of $J$ with respect to $\gamma_z^y, \gamma_w^y, \gamma_s^y, \gamma_0^y$ $\gamma_z^x, \gamma_s^x, \gamma_0^x, \gamma_z^z, \gamma_0^z, \gamma_s^s, \gamma_0^s$ and equates to 0 to solve for the solution of these parameters.

> **Remark 1.** *By using the above additive kernel functions, the form of latent confounders is as follows*
>
> $$\mathbf{z}_t = \gamma_0^z(\boldsymbol{\beta}^z)^\top \mathbf{1}_{TL} + \sum_{j=1}^{T}\sum_{l=1}^{L} \gamma_z^z((\boldsymbol{\beta}^z)^\top)_{:,jl}\phi(\mathbf{z}_{t-1}, \mathbf{z}_j^l) + \boldsymbol{\epsilon}_t, \quad \boldsymbol{\epsilon}_t \sim \mathsf{N}(\mathbf{0}, \mathbf{I}_{d_z}),$$
>
> *where* $((\boldsymbol{\beta}^z)^\top)_{:,jl}$ *denotes the jl-th column of* $(\boldsymbol{\beta}^z)^\top$. *Thus the form of* $\boldsymbol{\alpha}_0$, $\boldsymbol{\alpha}_{jl}$ *and* $k_z(\cdot,\cdot)$ *stated in the synthetic experiments (Section 5.2 in the main text) are* $\boldsymbol{\alpha}_0 = \gamma_0^z(\boldsymbol{\beta}^z)^\top \mathbf{1}_{TL}$, $\boldsymbol{\alpha}_{jl} = \gamma_z^z((\boldsymbol{\beta}^z)^\top)_{:,jl}$, *and* $k_z(\cdot,\cdot) = \phi(\mathbf{z}_{t-1}, \mathbf{z}_j^l)$.

# 7 SIMULATING SYNTHETIC DATA

In this section, we present how the synthetic datasets were generated in Section 5.2 of the main text. To simulate the data, we use the following Eqs. (2)-(6).

$$\mathbf{z}_t = f_z(\mathbf{z}_{t-1}) + \boldsymbol{e}_t, \qquad \boldsymbol{e}_t \sim \mathsf{N}(\mathbf{0}, \sigma_z^2 \mathbf{I}_{d_z}), \qquad (2)$$

$$\mathbf{s}_t = f_s(\mathbf{s}_{t-1}) + \boldsymbol{o}_t, \qquad \boldsymbol{o}_t \sim \mathsf{N}(\mathbf{0}, \sigma_s^2 \mathbf{I}_{d_s}), \qquad (3)$$

$$\mathbf{x}_t = f_x(\mathbf{z}_t, \mathbf{s}_t) + \boldsymbol{r}_t, \qquad \boldsymbol{r}_t \sim \mathsf{N}(\mathbf{0}, \sigma_x^2 \mathbf{I}_{d_x}), \qquad (4)$$

$$y_t = f_y(w_t, \mathbf{z}_t, \mathbf{s}_t) + v_t, \qquad v_t \sim \mathsf{N}(0, \sigma_y^2), \qquad (5)$$

$$w_t = \mathbb{1}(\varphi(f_w(\mathbf{z}_t, \mathbf{s}_t)) \geq u_t), \qquad u_t \sim \mathsf{U}[0, 1], \qquad (6)$$

where we set the ground truth hyperparameters $\sigma_z, \sigma_s, \sigma_x, \sigma_y$ to 1 and the dimension of $\mathbf{s}_t, \mathbf{z}_t, \mathbf{x}_t$ are 2. We specify the ground truth for each of the functions $f_z, f_s, f_x, f_y, f_w$ as neural networks whose specification are as follows:

- For $f_z(\mathbf{z}_{t-1})$: Input (2 dims) $\rightarrow \{\text{Hidden layer } i \text{ (10 dims)} \rightarrow \text{LeakyReLU}\}_{i=1}^{j} \rightarrow$ Output (2 dim),
- For $f_s(\mathbf{s}_{t-1})$: Input (2 dims) $\rightarrow \{\text{Hidden layer } i \text{ (10 dims)} \rightarrow \text{LeakyReLU}\}_{i=1}^{j} \rightarrow$ Output (2 dim),
- For $f_x(\mathbf{z}_t, \mathbf{s}_t)$: Input (4 dims) $\rightarrow \{\text{Hidden layer } i \text{ (10 dims)} \rightarrow \text{LeakyReLU}\}_{i=1}^{j} \rightarrow$ Output (2 dim),
- For $f_w(\mathbf{z}_t, \mathbf{s}_t)$: Input (4 dims) $\rightarrow \{\text{Hidden layer } i \text{ (10 dims)} \rightarrow \text{LeakyReLU}\}_{i=1}^{j} \rightarrow$ Output (1 dims),
- For $f_y(\mathbf{z}_t, w_t, \mathbf{s}_t)$: We set the ground truth $f_y(\mathbf{z}_t, w_t, \mathbf{s}_t) = (1 - w_t)f_{y_0}(\mathbf{z}_t, \mathbf{s}_t) + w_t f_{y_1}(\mathbf{z}_t, \mathbf{s}_t)$, where $f_{y_0}(\mathbf{z}_t, \mathbf{s}_t)$ and $f_{y_1}(\mathbf{z}_t, \mathbf{s}_t)$ are two different neural networks
  - $f_{y_0}(\mathbf{z}_t, \mathbf{s}_t)$: Input (4 dims) $\rightarrow \{\text{Hidden layer } i \text{ (10 dims)} \rightarrow \text{LeakyReLU}\}_{i=1}^{j} \rightarrow$ Output (1 dims),
  - $f_{y_1}(\mathbf{z}_t, \mathbf{s}_t)$: Input (4 dims) $\rightarrow \{\text{Hidden layer } i \text{ (10 dims)} \rightarrow \text{LeakyReLU}\}_{i=1}^{j} \rightarrow$ Output (1 dims),

where $j$ in the above specification is the number of hidden layers. Using different numbers of the hidden layers, i.e., $j = 2, 4, 6$, we construct three synthetic datasets, namely TD2L, TD4L, and TD6L. For these three datasets, we sample the latent confounder variable $\mathbf{z}$ satisfying Eq. (2).

We also construct another dataset, TD6L-iid, that uses six hidden layers but with the *iid* latent confounder variable, i.e., $\mathbf{z}_a \perp\!\!\!\perp \mathbf{z}_b$, and $\mathbf{s}_a \perp\!\!\!\perp \mathbf{s}_b$, $\forall a, b \in \{1, 2, ..., T\}$. In particular, the simulation of $\mathbf{z}_t$ and $\mathbf{s}_t$ in TD6L-iid was done using the following Eqs. (7) and (8)

$$p(\mathbf{z}_t|\mathbf{z}_{t-1}) = \mathsf{N}(\mathbf{z}_t|\mu_z, \sigma_z^2 \mathbf{I}_{d_z}), \qquad (7)$$

$$p(\mathbf{s}_t|\mathbf{s}_{t-1}) = \mathsf{N}(\mathbf{s}_t|\mu_s, \sigma_s^2 \mathbf{I}_{d_s}), \qquad (8)$$

where we set the ground truth $\mu_z = \mu_s = 1$. Simulation of the other variables remain the same, i.e, we use Eqs. (4)-(6) for simulation of $\mathbf{x}_t, y_t$ and $w_t$.

Finally, for all the simulation, we discard $\mathbf{z}_t$ and only keep $y_t$, $w_t$, $\mathbf{x}_t$ and $\mathbf{s}_t$ as observed data for training the models. We summarize the steps to simulate the datasets in Algorithm 1. In each step of Algorithm 1, we also save $f_{y_0}(\mathbf{z}_t, \mathbf{s}_t)$ and $f_{y_1}(\mathbf{z}_t, \mathbf{s}_t)$ for analysing the performance of each method.

**Algorithm 1:** Simulate synthetic dataset

1  **begin**
2      $\mathbf{z}_0 \sim \mathsf{N}(\mathbf{0}, \mathbf{I}_2)$;
3      $\mathbf{s}_0 \sim \mathsf{N}(\mathbf{0}, \mathbf{I}_2)$;
4      **for** $t := 1$ *to* $T$ **do**
5          Sample $\mathbf{z}_t$ using Eq. (2) for DT$j$L or Eq. (7) for DT6L-iid;
6          Sample $\mathbf{s}_t$ using using Eq. (3) for DT$j$L or Eq. (8) for DT6L-iid;
7          Sample $\mathbf{x}_t$ using Eq. (4);
8          Sample $w_t$ using Eq. (6);
9          Sample $y_t$ using Eq. (5);
10     **return** $\mathbf{w}, \mathbf{y}, \mathbf{x}, \mathbf{s}$;



\end{document}